\documentclass{article}

     \usepackage[preprint]{neurips_2020}

\usepackage{neurips_2020}

\usepackage[utf8]{inputenc} 
\usepackage[T1]{fontenc}    
\usepackage{hyperref}       
\usepackage{url}            
\usepackage{booktabs}       
\usepackage{amsfonts}       
\usepackage{nicefrac}       
\usepackage{microtype}      
\usepackage{amsmath}
\usepackage{amssymb}
\usepackage{amsthm}
\usepackage{tikz}
\usepackage{graphicx}
\usepackage{dsfont}
\usepackage{mathrsfs}
\usepackage{algorithmic}
\usepackage{algorithm}
\newtheorem{theorem}{Theorem}[section]
\newtheorem{assumption}{Assumption}[section]
\theoremstyle{theorem}
\newtheorem{definition}{Definition}[theorem]

\newtheorem{corollary}{Corollary}[theorem]

\newtheorem{lemma}{Lemma}[theorem]

\title{PAC Bounds for Imitation and Model-based Batch Learning of Contextual Markov Decision Processes}

\author{%
  Yash Nair\\
  Harvard University\\
  Cambridge, MA 02138 \\
  \texttt{yashnair@college.harvard.edu} \\
   \And
   Finale Doshi-Velez\\
   Harvard University\\
   Cambridge, MA 02138 \\
   \texttt{finale@seas.harvard.edu} \\
}

\begin{document}

\maketitle

\begin{abstract}
  We consider the problem of batch multi-task reinforcement learning with observed context descriptors, motivated by its application to  personalized medical treatment. In particular, we study two general classes of learning algorithms: direct policy learning (DPL), an imitation-learning based approach which learns from expert trajectories, and model-based learning. First, we derive sample complexity bounds for DPL, and then show that model-based learning from expert actions can, even with a finite model class, be impossible. After relaxing the conditions under which the model-based approach is expected to learn by allowing for greater coverage of state-action space, we provide sample complexity bounds for model-based learning with finite model classes, showing that there exist model classes with sample complexity exponential in their statistical complexity. We then derive a sample complexity upper bound for model-based learning based on a measure of concentration of the data distribution. Our results give formal justification for imitation learning over model-based learning in this setting.
\end{abstract}

\section{Introduction}
Families of context-dependent tasks are common in many real-world settings. For example, controlling a UAV might depend on factors such as the parameters of the specific UAV's weight and wingspan.  After successfully controlling several different UAVs, one might hope to be able to control a new UAV quickly.  Similarly, managing hypotension well may depend on some specific properties of the patient; after treating many distinct patients, one may hope to manage a new patient well.

The question of efficiently learning a collection of related, context-dependent tasks has been studied in the reinforcement learning  (RL) literature under many names such as lifelong RL, multi-task RL, and, more generally, transfer learning (see, e.g., \citet{DBLP:journals/corr/abs-1710-03850}, \citet{d2019sharing}), and \citet{taylor2009transfer}). Even more specifically, the question of learning to \textit{generalize} from a collection of data has been considered, for example, in \citet{brunskill2013sample} and \citet{DBLP:journals/corr/abs-1108-6211}. These works consider the problem in an online setting and develop algorithmic contributions in the batch setting, respectively.

In this work, motivated by the problem of treating patients with personalized strategies, we consider the following setting and question: Suppose we are given a batch of trajectories obtained from experts in multiple contexts, where each context's transition is parametrized by some observed parameter $\theta$---a framework called a Contextual MDP (CMDP).  Will it be more sample efficient to directly learn a policy from these data (that is, imitate the expert), or to learn the transition function, parametrized by $\theta$, and then plan according to it?  We derive upper and lower sample-complexity bounds for direct policy learning (DPL); our upper bound for DPL is polynomial in all the relevant parameters.  

Along the way, we prove impossibility results for learning certain transitions in the model-based paradigm, while our sample complexity upper bound for direct policy learning holds in a more general sense.  Next, we show that, under a relaxed data generation process which affords greater coverage of state-action space, their exist hard families of CMDPs for model-based learning; our results extending those of \citet{chen2019information} and \citet{krishnamurthy2016pac} to the multiple context setting. Finally, we derive a distribution-dependent sample complexity upper bound for model-based learning.

Our theory provides a formal justification for why imitation may be more successful than model-based learning in these settings, confirming trends observed in several more empirical and application-oriented works, including \citet{yao2018direct} and \citet{yang2019single}. In particular, we note that while both DPT and model-based learning depend linearly on the complexity of the hypothesis class, model-based learning requires an extra dependence on the concentratability of the data distribution which can, in some cases, render learning hard. Our results extend the sample complexity bounds derived for data collected online in \citet{brunskill2013sample} to the batch imitation learning and model-based context, as well as some of the single MDP results of \citet{piot2013learning} to CMDPs. 

\section{Related Work}

Some papers have considered the problem when the context is observable; that is, the learner is allowed to see some high-level labeling of each MDP in the training set as well as that of the target MDP before being asked to return a policy for the target MDP. In particular, \citet{DBLP:journals/corr/abs-1710-03850} describe the problem as zero-shot transfer learning and provide an algorithm which, under certain linearity assumptions regarding the task descriptions, is able to perform zero-shot transfer on large classes of MDPs. While they consider the problem both empirically and theoretically, they provide only convergence results and not a finite-sample analysis. This framework has also been studied in a more applied setting by \citet{sohn2018hierarchical}, who, under the further assumption that the learner is given access to not only task descriptions, but also a graph describing relationships among tasks, embed the graph and use a non-parametric gradient-based policy to obtain a policy for the target MDP. Both works differ most notably from ours in that we make no assumption on the structure of the task descriptors nor the relationships between them and we also derive sample complexity results.

The problem has also been cast as an imitation learning problem. \citet{osa2018algorithmic} describe \textit{behavior cloning}, the analog of DPL on a single MDP. They describe various algorithms and supervised learning techniques with which the learner can use to learn the expert policy, but do not give any sample complexity bounds. Furthermore, \citet{piot2013learning} consider Apprenticeship Learning, whereby the learner has access to a set of states and expert actions and is tasked with learning a policy. They upper bound the difference in the value function of the learner with that of the expert in terms of the classification error of the learner's hypothesis. Our work most notably differs from these two in that we consider the contextual setting: DPL reduces to these works, in the case when $|\Theta| = 1$, and thus only a single MDP is being considered.

Other works have viewed the problem in the domain where the task descriptors are not available to the learner (i.e. the parameter $\theta$ is latent). In particular, \citet{DBLP:journals/corr/abs-1108-6211} consider the problem of transfer learning between MDPs in the batch setting, where they design the \textit{All-sample Transfer} and \textit{Best Average Transfer} algorithms, both $Q$-value approximation algorithms designed to learn from a batch of trajectories to plan on a new MDP with potentially different transitions. We approach the problem of model-based learning in a different way from their approach, and instead extend off the analysis of \citet{chen2019information}. In the case where transfer is not considered, \citet{sun2018model} consider the model-based approach on contextual decision processes (CDPs). They show a $\textnormal{poly}(L, 1/\epsilon, \log(1/\delta))$ sample complexity is attainable for the model-based approach, in the batch setting. Our problem of transfer learning on multiple MDPs reduces to theirs in the case when $|\Theta| = 1$. Finally, \citet{yao2018direct} consider both a direct policy approach to the unobservable problem---direct policy transfer (DPT)---as well as a model-based approach in an empirical setting. They evaluate both learning strategies on a 2D Navigation task, Acrobot, and a simulated HIV treatment domain, and find that, empirically, DPT is more sample-efficient than the model-based approach.

\section{Notation and Background}

\paragraph{CMDPs}
A CMDP is a tuple $\langle S, A, R, T, L, \gamma, \Theta, \mathcal{P}_{s_0}, \mathcal{P}_\Theta \rangle$ where $S \subset \mathbb{R}^n$ and $A$ denote the state and (discrete) action space, respectively; $R:S\times A\rightarrow \mathbb{R}$ is a reward function; $\gamma$, a discount factor used in evaluating long-term return; $\Theta \subset \mathbb{R}^d$; and $T: S\times A\times \Theta\rightarrow \Delta(S),$ is a conditional density over next states given the current state, action, and value of $\theta$ parametrizing the transition.  The initial state $s_0$ is drawn from $\mathcal{P}_{s_0}$; at the start of an episode, the task parameter $\theta$ is drawn from some $\mathcal{P}_\theta$. This is a slight simplification of the standard definition of CMDP of \citet{hallak2015contextual} in that we assume the reward and initial distribution remain the same.

\paragraph{Setting} In this paper, we shall consider the case in which the learner is given a batch of $m$ trajectories of length $L$, each labeled with its associated context $\theta$; that is, for each trajectory, the learner is allowed to see the value of $\theta$ corresponding to the MDP from which it was drawn.  There is one trajectory per parameter setting $\theta$, corresponding to real settings in which one only gets to treat each patient once. This is in contrast to the setting of Hidden Parameter MDPs (HiP-MDPs), introduced by \citet{doshi2016hidden}, in which the parameter $\theta$ is latent. We will further assume that the learner has access to the reward function, $R$.

 Throughout this paper, we will assume that the trajectories are formed via following a deterministic (but possibly time-dependent) expert policy $\pi$ that is $\alpha$-optimal; that is, $v^L_\pi \geq v^L_{\pi^*} - \alpha,$ where $\pi^*$ is the optimal deterministic time-dependent policy and $v^L_\pi:= \mathbb{E}[\sum_{l=0}^{L-1}r_l]$ is the expected undiscounted value associated with the first $L$ rewards, where here the expectation also includes the randomness with respect to the draw of $\theta \sim \mathcal{P}_\Theta$. We also define $V^L_\pi(s;\theta) = \mathbb{E}[\sum_{l=0}^{L-1}r_l|s_0=s,\theta]$.



\paragraph{Goal} Under this data generation process, the learner's goal is to return a policy $\hat{\pi}: S\times \Theta\times \{0, \ldots, L-1\} \rightarrow A,$ such that its value $v^l_{\hat{\pi}}$ is maximized. Specifically, we will define the error in value of a hypothesis below:
\begin{definition}
We define the error (of reinforcement learning) of a policy $\hat{\pi}$ to be the difference in undiscounted value $v^L_{\pi^*} - v^L_{\hat{\pi}}$.
\end{definition}
Formally, the goal is thus to return a policy with small error in value. 

Following learning theory terminology, we shall also refer to $\hat{\pi}$ as the \emph{hypothesis} returned by the learner.

\section{Sample Complexity Bounds for DPL}
We now turn to our learning problems.  One approach to learning a hypothesis $\hat{\pi}$ above is simply to treat the problem as a supervised learning problem and directly learn the association between the inputs---the states $s$ and the task parameters $\theta$---and the expert's action $a$.  In the following, we assume the learner is allowed to return any hypothesis from some hypothesis class, $\mathcal{H}$ with \[ h: S\times \Theta\times\{0, \ldots, L-1\} \rightarrow A, \forall h \in \mathcal{H}.\] In particular, just as in \citet{yao2018direct}, who consider a related algorithm called Direct Policy Transfer, we assume that DPL is agnostic to the reward sequence of the expert. 

\subsection{DPL Sample Complexity Upper Bound}
We now derive a sample complexity upper bound for DPL. Our analysis is similar to that of the standard agnostic PAC learning upper bound, except that, in this setting, the batch of data are not i.i.d, but rather, come from a Markov chain. This, however, can be remedied, by simply replacing one of the key concentration inequalities in the standard setting (McDiarmid's) with an analogous concentration inequality which applies to Markov chains, shown in \citet{paulin2015concentration}. We define notions of classification error in this setting.

\begin{definition}[$n$th Marginal Error]
Let $n \in \mathbb{N}$ be at most $L-1$. Define the $n$th marginal error of a hypothesis $h$ to be \[\mathcal{L}^{marginal,n}_{\pi}(h) = P_{\pi}(h(s_n,n) \neq \pi(s_n,n)),\] where the probability is being taken with respect to the randomness of drawing $\theta \sim \mathcal{P}_\Theta$, drawing $s_0 \sim \mathcal{P}_{s_0}$, and the randomness from following the deterministic policy $\pi$ for $n$ steps under the (potentially stochastic) transitions indexed by $\theta$.
\end{definition}
We now arrive at a definition of true error in this setting. 

\begin{definition}
Define the true error of the hypothesis $h$ to be \[\mathcal{L}^{L}_{\mathcal{D}, \pi}(h) = \frac{1}{L}\sum_{l=0}^{L-1}\mathcal{L}_{\pi}^{marginal, l}(h).\]
\end{definition}
Finally, the definition of empirical error is same as in the standard setting:
\begin{definition}
Letting $S$ consist of $mL$ tuples, as outlined in the data generation procedure above, we define the empirical error to be \[\mathcal{L}^{m,L}_{S}(h) = \frac{|\{(i,j) \in \{1, \ldots, m\}\times \{0, \ldots, L-1\}: h(s_j^{\theta_i}, \theta_i, j) \neq \pi(s^{\theta_i}_j, \theta_i, j)\}|}{mL}.\]
\end{definition}

With these definitions in place, we are able to derive a sample complexity of classification upper bound for DPL using any hypothesis class $\mathcal{H}$. Our analysis follows the standard agnostic PAC analysis (see, e.g., \citet{shalev2014understanding}) but applied to the non-i.i.d. setting, in which we use a concentration result of \citet{paulin2015concentration}. We state the agnostic sample complexity of classification upper bound in terms of the Natarajan dimension of the hypothesis class $\textnormal{Ndim}(\mathcal{H})$ (for proof of the theorem and definition of Natarajan dimension, see Appendix Section A.1)

\begin{theorem}
Let the concept class $\mathcal{H}$ have Natarajan dimension $d$. There exists a learning algorithm $\mathcal{A}$ such that for any distribution over the data, there exists $m$ with \[m = O\left(\frac{d}{\epsilon^2}\left(\log\left(\frac{d}{\epsilon L}\right) + \frac{d}{L}(\log(L) + \log(|A|)) + L^2\log(1/\delta)\right)\right),\] such that if $\mathcal{A}$ receives at least $m$ $L$-long trajectories in the batch, then $\mathcal{A}$ returns a hypothesis in $\mathcal{H}$ which has classification error at most $\epsilon$ greater than the true error minimizer of $\mathcal{H}$ with probability at least $1-\delta$.
\end{theorem}

Before considering how the classification error of $\pi$ affects its error in value, we make two standard assumptions regarding the hypothesis class $\mathcal{H}$ as well as the reward structure: 

\begin{assumption}[Realizability]
The hypothesis class $\mathcal{H}$ contains the expert's labeling, $\pi$.
\end{assumption}

\begin{assumption}
All rewards are in the range $[0,1]$.
\end{assumption}

We use an extension of a result of \citet{ross2010efficient} who bound the error of a time-independent policy in terms of its classification error for a single sequential decision-making task. By carefully comparing the marginal distribution of states under $\hat{\pi}$ and $\pi$ at each time step, they are able to bound the difference in values by $L^2\epsilon$, where $\epsilon$ is the classification error of $\hat{\pi}$ with respect to $\pi$. The proof in our multi-task setting with time-dependent policies follows theirs (for details, see \citet{ross2010efficient}). With this result as well as our above sample complexity of classification bound, we now show that DPL requires only a $\textit{poly}(\textnormal{Ndim}(\mathcal{H}), \frac{1}{\epsilon}, \log(\frac{1}{\delta}), L, \log(|A|))$ number of samples.

\begin{theorem}
Let the concept class $\mathcal{H}$ have Natarajan dimension $d$. There exists a learning algorithm $\mathcal{A}$ such that for any distribution over the data, there exists $m$ that is 
\[O\left(\frac{L^4d}{\epsilon^2}\left(\log\left(\frac{Ld}{\epsilon}\right) + \frac{d}{L}(\log(L) + \log(|A|)) + L^2\log(1/\delta)\right)\right)\]
receives at least $m$ $L$-long trajectories in the batch, then $\mathcal{A}$ returns a hypothesis in $\mathcal{H}$ which has error (in terms of undiscounted value) at most $\epsilon + \alpha$ with probability at least $1-\delta$.
\end{theorem}

\subsection{DPL Sample Complexity Lower Bound}
We derive a DPL lower bound by constructing a family of CMDPs for which the problem reduces to a standard PAC learning problem that must be learned to error $\epsilon/L$ with confidence $\delta$. To do so, we essentially put all the decision making power on the action taken at the first state. That is, let $\mathcal{P}_{s_0}$ be fully concentrated on the state $s_0$, and let all transitions be deterministic so that there are $|A|$ distinct potential next states. For each value of $\theta$, call the state $s$, that satisfies $T(s|s_0, \pi(s_0,\theta), \theta) = 1$ good, and call the other $|A|-1$ states bad. The good state satisfies the condition that it and all subsequent states give reward $1$ independent of the actions taken in them, while any bad state and all subsequent states give reward $0$, again, independently from actions. The learning problem is thus reduced to learning the first action to error at most $\epsilon/L$ with probability at least $1-\delta$. Choosing $\mathcal{H}$ so that there is a shattered set of the form $\{(\theta_1, s_0), \ldots, (\theta_m, s_0)\}$, the following lower bound follows immediately from the standard i.i.d PAC lower bound.

\begin{theorem}
There exist a family of CMDPs, a hypothesis class, $\mathcal{H}$ of Natarajan dimension $d$, and value of $m$ with \[m = \Omega\left(\frac{L(d + \log(1/\delta))}{\epsilon}\right),\] such that any learning algorithm given at most $m$ $L$-long trajectories returns a policy whose error (in terms of value) is at most $\epsilon$ with probability at most $1-\delta$.
\end{theorem}

Our results thus give a $\tilde{O}\left(\frac{L^2d}{\epsilon} + \frac{L^5}{\epsilon}\right)$ separation between upper and lower bounds, in particular, highlighting how poorly the bounds scale with $L$, but how slowly they separate with $\delta$ and $A$; we discuss this further in Section 6.

\section{Model-based Approach}
In contrast to DPL, the model-based approach does not attempt to directly learn the expert policy, but rather attempts to learn the transition function parametrized by $\theta$, and then plan according to the transition and reward function. One might believe that learning models would be more general than trying to directly learn policies, because one can use them to explore counterfactuals.  Indeed, model-based learning is often the go-to approach in low data regimes (see, e.g. \citet{rasmussen2003gaussian} versus \citet{deisenroth2013gaussian}, \citet{kamthe2017data}, \citet{kocijan2004gaussian}, \citet{ko2007gaussian}). However, we first show that the paradigm of learning from expert actions is, in general settings, impossible---even if direct policy learning is possible in these settings. Thus, we relax the data generation process under which the model-based method is expected to learn and introduce a process which allows for greater coverage of $S\times A$. 

Throughout our analysis, for simplicity, we will assume the model-based approach has access to an oracle called \texttt{PLAN}, which, upon receiving a transition function, $T$ for a single MDP, and reward function, $R$ (which, as mentioned in Section 3, the learner has access to), returns an optimal deterministic (possible time-dependent) policy under $T$ and $R$ in that MDP. Hence, if the model-based algorithm returns a hypothesis transition function $h$, its corresponding value is simply $v^L_{\texttt{PLAN}(h(\theta),R)}$, where $h(\theta)$ denotes the transition function restricted to $\theta$. Note that evaluating a model-based algorithm in this sense implicitly defines the policy, $\hat{\pi}_h$, corresponding to the hypothesis transition function $h$, returned as $\hat{\pi}_h(s,\theta,l) = (\texttt{PLAN}(h(\theta), R))(s,l)$.

\subsection{Impossibility of Model-Based Learning via Expert Actions}
The primary issue a model-based approach having access only to an expert's trajectory is the lack of coverage of state-action space. In particular, it may be the case that a large portion of the learner's hypothesis class always agree on the subset of state-action space traversed by the expert, making it hard---and, as we show below, in some cases impossible---to output a low-error hypothesis with high confidence. We now give a construction formalizing the above which demonstrates an infinite sample complexity for the model-based approach. Under our construction, the learner is given a finite model hypothesis class from which to return its hypothesis transition function. This eliminates the possibility of the infinite sample complexity being explained away by the hypothesis class having infinite statistical complexity, and instead highlights the fundamental difficulty model-based learning faces when attempting to learn from a single expert's actions.

\begin{theorem}
There exist classes of CMDPs which, if the learner must learn from expert trajectories, require infinite sample complexity even with a finite hypothesis class containing the true model.
\end{theorem}
\begin{proof}
Our construction is inspired by that of \citet{chen2019information} who use a similar setup to derive an information-theoretic lower bound in the single MDP batch setting. First, let $\Theta = [|A|^{L-1}-1]$ and $\mathcal{P}_\Theta = \textnormal{Unif}(\Theta)$. Now, let a CMDP be represented as a rooted complete tree with branching factor $|A|$ and depth $L$. Each non-leaf node represents a distinct state giving $0$ reward, while each leaf node represents two possible states; we furthermore restrict $\mathcal{P}_{s_0}$ to concentrate all its mass on some single initial state. Taking action $a_{|A|}$ (the rightmost action) from the rightmost node in the penultimate layer transitions to a state with reward $0$ with probability $\frac{1 - \epsilon}{2}$ and reward $1$ with probability $\frac{1 + \epsilon}{2}$. Finally, let all other edges leaving the penultimate layer, except the $\theta$th edge from the left, transition to a state with reward $0$ with probability $1/2$ and a state with reward $1$ with probability $1/2$, while the $\theta$th edge transitions to a state with reward $0$ with probability $\frac{1}{2} - \frac{4}{3}\epsilon$ and reward $1$ with probability $\frac{1}{2} + \frac{4}{3}\epsilon$ (so our result will hold for $\epsilon < \frac{3}{8}$). We will call this edge, leading to the highest rewarding leaf, in expectation, the special edge, and we will let the expert's policy be that which continually chooses the rightmost action---notice that the expert is $\frac{5}{6}\epsilon$-optimal. Call a CMDP of the above form a tree CMDP: Notice that a tree CMDP is fully characterized by the map taking $\theta$ to the special edge---in this case, the identity map. It is thus clear that any permutation of $\Theta$ will yield a different special edge map (i.e. the permutation $\sigma$ says that $\theta$ corresponds to the MDP with special edge given by the $\sigma(\theta)$th edge of the penultimate layer). So, the learner is given as its hypothesis model class two hypotheses: the identity permutation and some derangement $\sigma$. Noticing that any algorithm must determine precisely the target CMDP in order to return an $\epsilon$-optimal hypothesis (since otherwise, \texttt{PLAN} gives a policy with value $1/2$ because if $h \in \mathcal{H}$ is not equal to the target concept, it will disagree with the target concept on every value of $\theta$, since we chose $\sigma$ to be a derangement) and that the expert's trajectory gives no information about which member of $\mathcal{H}$ is the target CMDP, any learning algorithm returns an $\epsilon$-optimal hypothesis with probability at most $\frac{1}{2}$, and, in particular, cannot guarantee $\epsilon$-optimality with $\delta$ confidence for any $\delta, \epsilon < \frac{1}{2}$, thus indicating the impossibility of learning this model class.
\end{proof}

The above construction illustrates that learning from a single expert's trajectory can be impossible for any model-based learning algorithm. In particular, while DPL, with sufficiently many samples, is able to return a policy which is nearly $\frac{5}{6}\epsilon$-optimal, any model-based learning algorithm will, with probability at least $\frac{1}{2}$ (regardless of the size of $m$), return a policy which has suboptimality larger than $\epsilon$. This highlights the key difficulty with model-based learning in a CMDP from an expert's trajectories: while DPL can always achieve suboptimality arbitrarily close to that of the expert (as long as the Natarajan dimension of its hypothesis class is finite), the same cannot be said for model-based learning even when its hypothesis class is finite.

\subsection{Hardness of Model-based Learning under strictly-positive visitation distributions}
The impossibility result above motivates the use of the following more standard framework under which we expect batch model-based approaches to learn:

\begin{definition}[Model-based Learning Data Generation Process]
Let $\mu$ be some distribution over $S\times A$ which assigns non-zero mass/density to every $(s,a) \in S\times A$. For each value of $\theta \sim \mathcal{P}_\Theta$, the model-based approach draws $L$ pairs $(s,a) \stackrel{i.i.d}{\sim} \mu$ and, for each pair, draws $s' \sim T(\cdot|s,a,\theta)$ and $r = R(s,a)$. The model-based approach then has access to each of these one-step trajectories, labelled by $\theta$. We also make the assumption that every element of $S\times A$ is reachable in at most $L$ steps.
\end{definition}

Under the framework defined in Definition 5.1.1, we show that there are still classes for which model-based learning is hard. In fact, the construction of \citet{chen2019information} for single MDPs gives a lower bound in our setting as well since we can simply consider the CMDP which concentrates all its mass on a single value of $\theta$, thus reducing to the single MDP case. Their construction is very similar to the above tree construction: States are the nodes of a complete tree with branching factor $|A|$, and all leaf nodes give $\textnormal{Bern}(1/2)$ rewards while the leaf node corresponding to the special edge gives $\textnormal{Bern}(1/2 + 3\epsilon/2)$ reward. 
In our setting, the construction of \citet{krishnamurthy2016pac} yields a sample complexity lower bound of $\Omega\left(\frac{|A|^L}{L\epsilon^2}\right)$ when active exploration is allowed. This sample complexity is super-polynomial in statistical complexity of the hypothesis class, which has cardinality $|A|^L$. In Appendix Section A.2, we give a construction of a family of MDPs, motivated by contextual bandits and the constructions of \citet{krishnamurthy2016pac} and \citet{10.1137/S0097539701398375}, which yields a sample complexity lower bound of $\Omega\left(\frac{|A|}{\epsilon^2}\right)$, allowing exploration for a hypothesis class of cardinality $|A|$. Our bound, on the surface, is asymptotically lower than that of \citet{krishnamurthy2016pac}; however, since the hypothesis class as well as the size of each CMDP in the class have size only $O(|A|)$, our bound exhibits a stronger dependence on the size of the model class. That is, while the original bound of \citet{krishnamurthy2016pac} precludes $\textit{poly}(\log(|\mathcal{H}|), \frac{1}{\delta}, \frac{1}{\epsilon})$ sample complexity, our second bound gives a $\Omega\left(\frac{|\mathcal{H}|}{\epsilon^2}\right)$ dependence on the size of the hypothesis class, rather than the immediate $\Omega\left(\frac{|\mathcal{H}|}{L\epsilon^2}\right)$ of \citet{krishnamurthy2016pac}.

The above constructions show how the sample complexity for model-based learning can scale poorly with both horizon as well as the size and statistical complexity of the hypothesis class, but fail to show that model-based learning can scale poorly with $|\Theta|$ when it is finite. \citet{modi2017markov}, who consider online CMDPs, essentially suggest constructing hard CMDPs by making the MDP for each context, which is drawn uniformly from $\Theta$, hard and disallowing any information corresponding to one context be useful to another. That is, having knowledge of the target transition function for all values in $\Theta \backslash \{\theta\}$ provides no information about the target transition for $\theta$. While this technique gives a generic way to increase any hard MDP lower bound by a multiplicative factor of $O(|\Theta|)$ in expectation for CMDPs, it unfortunately makes the hypothesis class have cardinality exponential in $|\Theta|$, therefore explaining away the factor of $O(|\Theta|)$ by the complexity of the hypothesis class. We give a construction of a class of CMDPs in Appendix Section A.2, extending off of \citet{krishnamurthy2016pac}, which yields the same lower bound, but does so in a way that scales linearly with $|\Theta|$, while having a hypothesis class of cardinality $|\Theta|$.

We now state lower bounds both in the setting where active exploration is allowed and also when it is not. The latter bound is in terms of, $C$, the concentratability coefficient of $\mu$, which measures how much $\mu$ covers reachable state-action pairs (for a definition, see Section 5.3.1).

\begin{theorem}
There exist hard families of CMDPs which are subject to the following sample complexity lower bounds (all are asymptotically at most $\Omega\left(\frac{|A|^L}{L\epsilon^2}\right)$): $\Omega\left(\frac{|A|^L}{L\epsilon^2}\right)$,  $\Omega\left(\frac{|\mathcal{H}|}{\epsilon^2}\right)$, and $\Omega\left(\frac{|\Theta|}{L\epsilon^2}\right)$ with $|\Theta| = |\mathcal{H}|$. Furthermore, when active exploration is not allowed, we have the following lower bound, in expectation, with respect to the randomness of drawing a leaf state from $\mu$: $\Omega\left(\frac{C|A|^L}{L\epsilon^2}\right)$
\end{theorem}

Thus, to summarize, we have shown not only that the sample complexity for model-based learning can grow poorly with $|\Theta|$ and $L$ even when $\mathcal{H}$ is not too large, but also that the sample complexity depends on the concentratability coefficient at least linearly.


\subsection{Model-based Learning Sample Complexity Upper Bound}

We now derive an upper bound for model-based learning to contrast the above lower bound and, in particular, show that the dependence on $C$ is in fact linear. To do so, we extend off the work of Chen \& Jiang (2019), who derive an upper bound for Fitted Q-Iteration (FQI). The FQI sample complexity upper bound then immediately yields a sample complexity upper bound for any model-based approach with a finite model class. In their approach, they consider an infinite-horizon discounted setting on a single MDP, whereas in our setting, we consider a finite-horizon undiscounted setting on CMDPs, in which time-dependent policies are allowed. First, we show how FQI can be applied in this setting, define a complexity measure which controls $\mu$, the distribution from which state-action pairs are drawn, and then state a sample complexity result for FQI from which a sample complexity upper bound for model-based learning follows as an immediate upper bound. 

\subsubsection{FQI on Finite-Horizon Undiscounted CMDPs}
We give a brief outline of finite-horizon FQI on a CMDP below; it is essentially the same as FQI on a single MDP except that Bellman backups are done with respect to the context, $\theta$. We first define this back-up and give the algorithm below:
\begin{definition}
Define the $l$th Bellman backup of $f: S\times\Theta\times \{0, \ldots, L\} \rightarrow \mathbb{R}$ with respect to $\theta$ to be \[(\mathcal{T}_{l}(\theta)f)(s,a) = R(s,a) + \mathbb{E}_{s' \sim T(\cdot|s,a,\theta)}V_f(s',\theta, l),\] where $V_f(s, \theta, l) = \max_{a \in A}f(s,a,\theta, l)$. 
\end{definition}

We will also assume FQI has access to a family of time-indexed $Q$ functions. That is we have a set $\mathcal{F}$ which contains $Q$-value functions of the form $Q: S\times A\times \Theta\times\{0, \ldots, L\}$, where $Q(s,a,\theta,l) = (\mathcal{T}_{l-1}(\theta)Q)(s,a)$ for $l \geq 1$, and $Q(s,a,\theta, 0) =0$. FQI on CMDPs operates in essentially the same way as the finite horizon case except that all backups and value functions are additionally parametrized by $\theta$ (see Appendix Section A.3 for pseudocode).

We now give the definition of admissible distribution and the assumption of concentratability of the data disitribution $\mu$ which extends that of Chen \& Jiang (2019). 

\begin{definition}[Admissible Distribution]
A conditional distribution $\nu$ over $S\times A$ given $\theta \in \Theta$ is said to be admissible if there exists $0 \leq l \leq L-1$ if there exists a possibly time-dependent stochastic policy $\pi$ such that $(\nu(\theta))(s,a) = P[s_l = s, a_h = a|\theta, s_0 \sim \mathcal{P}_{s_0}, \pi]$
\end{definition}

\begin{assumption}[Concentratability]
We assume that there exists some $C < \infty$ such that, for any admissible distribution $\nu$, \[\frac{(\nu(\theta))(s,a)}{\mu(s,a)} \leq C, \forall (s,a,\theta) \in S\times A\times\Theta.\]
\end{assumption}

With these definitions, we now state a sample complexity upper bound for FQI, and then derive, as an immediate corollary, a sample complexity upper bound for model-based learning; again, we assume realizability for the hypothesis class $\mathcal{H}$ and thus of the class $\mathcal{F}$. For proofs, see Appendix A.3.

\begin{theorem}[FQI Upper Bound]
There exists $m$ with \[m = O\left(\frac{CL^6\log(L|\mathcal{F}|/\delta)}{\epsilon^2}\right),\] such that if FQI receives at least $m$ samples of $L$ one-step trajectories under $\mu$, then it returns a policy with error at most $\epsilon$ with probability at least $1-\delta$.
\end{theorem}

\begin{corollary}[Model-based Upper Bound]
Given the finite hypothesis model class $\mathcal{H}$, there exists a model-based learning algorithm $\mathcal{A}$ and $m$ with \[m = O\left(\frac{CL^6\log(L|\mathcal{H}|/\delta)}{\epsilon^2}\right),\] such that if $\mathcal{A}$ receives at least $m$ samples of $L$ one-step trajectories under $\mu$, then it returns a policy with error at most $\epsilon$ with probability at least $1-\delta$.
\end{corollary}

\section{Discussion}
In this paper we investigate the sample complexities of an imitation learning-base approach for learning of CMDPs, DPL, as well as a model-based approach. We find that the upper bounds for each approach are, respectively, $\tilde{O}\left(\frac{L^4d}{\epsilon^2}\left(\frac{d}{L} + L^2\right)\right)$ and $\tilde{O}\left(\frac{CL^6}{\epsilon^2}\right).$ Our results indicate that DPL is, theoretically, more sound than model-based approaches in the sense that the latter scales with respect to the concentratability coefficient of the distribution $\mu$. In particular, while both upper bounds scale polynomially in all the relevant parameters---and, importantly, in the complexity of hypothesis class---our upper bound for model-based learning scales with $C$. As our lower bound for model-based learning shows, this additional dependence is, in fact, \textit{necessary}: Data distributions $\mu$ which concentrate low mass to regions of $S\times A$ which differentiate hypotheses are harder to learn under. This highlights the importance of the data generation process for model-based learning: When data is gotten from expert trajectories, model-based learning can be impossible even with finite hypothesis classes, but even when data is drawn i.i.d. from the distribution $\mu$, model-based learning depends greatly on the coverage of reachable state-action pairs.

We believe the following are primary interests for future work: Deriving general model-based sample complexity upper bounds which do not grow, even logarithmically with $|\mathcal{H}|$, but rather grow with some other complexity measure of the hypothesis class which can be finite even for infinite $\mathcal{H}$ (e.g. perhaps with an extension of \textit{witness rank} introduced in \citet{sun2018model}); investigating a tighter relationship between the upper and lower bounds for DPL, in particular, bounds whose degree of separation scales more slowly with $L$; and understanding the sample complexity of similar imitation learning and model-based algorithms in the unobserved parameter setting of HiP-MDPs.




\section{Broader Impact}
The primary real-world impact of this research---and, in fact, the application by which the authors were motivated---is to the world of healthcare. In particular, this work serves as a first-step to formally understanding the more complex framework of HiP-MDPs introduced in \citet{doshi2016hidden} to encapsulate learning to generalize from expert actions when the parametrization of the transition of each MDP is unknown and the parameter corresponding to each trajectory in the batch is unknown. In particular, the motivation for HiP-MDPs is to learn how to generalize from a healthcare professional's actions when treating patients whose response to the same treatment strategy may differ. Thus, we see our work as a first step in formally understanding what types of algorithms may be most sample-efficient for learning to generalize well in such settings.  That said, this work is theoretical in nature, and makes standard theory assumptions such as Markovianity under the definition of state $s$, existence of the true function with the hypothesis class, etc. Thus, while we provide theoretical foundations and insights for empirical observations, any application of our work to real settings should be mindful of the assumptions we make.

\begin{ack}
FDV acknowledges support from NSF RI-1718306. YN acknowledges support from HCRP. We also thank George Konidaris, Melanie Pradier, and Weiwei Pan for useful feedback as well as William Zhang for helpful discussions. 
\end{ack}

\nocite{*}
\bibliographystyle{plainnat}
\bibliography{references.bib}
\small

\newpage
\appendix
\section{Proofs}
\subsection{DPL Sample Complexity of Classification Upper Bound}
Our proofs below follow the analysis of \citet{shalev2014understanding} and use similar notation.
\begin{lemma}
Let $S$ be a batch of trajectories. Then $\mathbb{E}_{S}[\mathcal{L}_S^{m,L}(h)] = \mathcal{L}^{L}_{\mathcal{D}, \pi}(h)$.
\end{lemma}
\begin{proof}
Note that $$\mathcal{L}^{m,L}_{S}(h) = \frac{1}{mL}\sum_{i=1}^m\sum_{l=0}^{L-1}\mathds{1}_{\{h(s_l^{\theta_i}, \theta_i, l) \neq \pi(s_l^{\theta_i}, \theta_i, l)\}},$$ and notice that $$\mathbb{E}\left[\mathds{1}_{\{h(s_l^{\theta_i}, \theta_i, l) \neq \pi(s_l^{\theta_i}, \theta_i, l)\}}\right] = \mathcal{L}_{\pi}^{marginal, l},$$ and so $$\mathbb{E}\mathcal{L}^{m,L}_{S}(h) = \frac{1}{mL}\sum_{i=1}^m\sum_{l=0}^{L-1}\mathcal{L}_{\pi}^{marginal, l} = \frac{1}{L}\sum_{l=0}^{L-1}\mathcal{L}_{\pi}^{marginal, l} = \mathcal{L}^{m,L}_{\mathcal{D}, \pi}(h),$$ as desired.
\end{proof}

\begin{definition}
A training set, $S$, of $m$ $L$-long trajectories is $\epsilon$-representative if $$\sup_{h \in \mathcal{H}}|\mathcal{L}_{\mathcal{D}, \pi}^{L}(h) - \mathcal{L}_S^{m,L}(h)| \leq \epsilon.$$
\end{definition}

Now we proceed with a lemma regarding bounds on $\epsilon$-representativeness; it is essentially the same as that in \citet{shalev2014understanding} save for a few minor modifications.
\begin{lemma}
Let $S$ be a training set consisting of $m$ $L$-long trajectories. Then $$\mathbb{E}_S\left[\sup_{h \in \mathcal{H}}\left(\mathcal{L}^{L}_{\mathcal{D}, \pi}(h) - \mathcal{L}^{m,L}_S(h)\right)\right] \leq 2\mathbb{E}_S R(\mathcal{F}\circ S),$$ where $R$ denotes the Rademacher complexity and $$\mathcal{F}\circ S = \left\{\left(\mathds{1}_{\{h(s_0^{\theta_1}, \theta_1, 0) \neq \pi(s_0^{\theta_1}, \theta_1, 0)\}}, \ldots, \mathds{1}_{\{h(s_{L-1}^{\theta_m}, \theta_m, L-1) \neq \pi(s_{L-1}^{\theta_m}, \theta_m, L-1)\}}\right): h \in \mathcal{H}\right\}.$$
\end{lemma}
\begin{proof}
Let $S$ and $S'$ be two datasets both sampled according to the procedure above and recall that $\mathcal{L}^{L}_{\mathcal{D}, \pi}(h) = \mathbb{E}_{S'}\mathcal{L}^{m,L}_{S'}(h), \forall h \in \mathcal{H}$. Thus, noting that the supremum of the expectation is at most the expectation of the supremum, we have $$\sup_{h \in \mathcal{H}}\left(\mathcal{L}^{L}_{\mathcal{D}, \pi}(h) - \mathcal{L}^{m,L}_S(h)\right) = \sup_{h \in \mathcal{H}}\mathbb{E}_{S'}\left[\mathcal{L}^{m,L}_{S'}(h) - \mathcal{L}^{m,L}_S(h)\right]$$ $$\leq \mathbb{E}_{S'}\left[\sup_{h \in \mathcal{H}}(\mathcal{L}^{m,L}_{S'}(h) - \mathcal{L}^{m,L}_S(h))\right],$$ so taking expectations gives $$\mathbb{E}_S\left[\sup_{h \in \mathcal{H}}\left(\mathcal{L}^{L}_{\mathcal{D}, \pi}(h) - \mathcal{L}^{m,L}_S(h)\right)\right] \leq \mathbb{E}_{S,S'}\left[\sup_{h \in \mathcal{H}}(\mathcal{L}^{m,L}_{S'}(h) - \mathcal{L}^{m,L}_S(h))\right]$$ $$= \frac{1}{mL}\mathbb{E}_{S,S'}\left[\sup_{h \in \mathcal{H}}\sum_{i=1}^m\sum_{l=0}^{L-1}(\mathds{1}_{\{h(s_l^{\theta_i}, \theta_i, l) \neq \pi(s_l^{\theta_i}, \theta_i, l)\}} - \mathds{1}_{\{h(s_l^{\theta'_{i'}}, \theta'_i, l) \neq \pi(s_l^{\theta'_{i'}}, \theta'_i, l)\}})\right].$$ Now, notice that since $\theta_i$ and $\theta_i'$ are independent and identically distributed and $s^{\theta_i}_l$ and $s^{\theta_{i'}'}_l$ are independent and identically distributed, we see that $$\mathbb{E}\Big[\sup_{h \in \mathcal{H}}\Big((\mathds{1}_{\{h(s^{\theta_{i'}'}_l, \theta_i', l) \neq \pi(s^{\theta_{i'}'}_l, \theta_i', l)\}} - \mathds{1}_{\{h(s^{\theta_{i}}_l, \theta_i, l) \neq \pi(s^{\theta_{i}}_l, \theta_i, l)\}})$$ $$+ \sum_{k\neq i}\sum_{j \neq l}(\mathds{1}_{\{h(s^{\theta_{k'}'}_j, \theta_k', j) \neq \pi(s^{\theta_{k'}'}_j, \theta_k', j)\}} - \mathds{1}_{\{h(s^{\theta_{k}}_j, \theta_k, j) \neq \pi(s^{\theta_{k}}_j, \theta_k, j)\}}) \Big)\Big]$$ $$= \mathbb{E}\Big[\sup_{h \in \mathcal{H}}\Big((\mathds{1}_{\{h(s^{\theta_{i}}_l, \theta_i, l) \neq \pi(s^{\theta_{i}}_l, \theta_i, l)\}} - \mathds{1}_{\{h(s^{\theta_{i'}'}_l, \theta_i', l) \neq \pi(s^{\theta_{i'}'}_l, \theta_i', l)\}})$$ $$+ \sum_{k\neq i}\sum_{j \neq l}(\mathds{1}_{\{h(s^{\theta_{k'}'}_j, \theta_k', j) \neq \pi(s^{\theta_{k'}'}_j, \theta_k', j)\}} - \mathds{1}_{\{h(s^{\theta_{k}}_j, \theta_k ,j) \neq \pi(s^{\theta_{k}}_j, \theta_k, j)\}}) \Big)\Big],$$ so that if $\sigma_{i,l}$ is a random sign, we see that, by the law of total expectation, $$\mathbb{E}_{S,S', \sigma_{i,l}}\Big[\sup_{h \in \mathcal{H}}\Big(\sigma_{i,l}(\mathds{1}_{\{h(s^{\theta_{i'}'}_l, \theta_i', l) \neq \pi(s^{\theta_{i'}'}_l, \theta_i', l)\}} - \mathds{1}_{\{h(s^{\theta_{i}}_l, \theta_i, l) \neq \pi(s^{\theta_{i}}_l, \theta_i, l)\}})$$ $$+ \sum_{k\neq i}\sum_{j \neq l}(\mathds{1}_{\{h(s^{\theta_{k'}'}_j, \theta_k', j) \neq \pi(s^{\theta_{k'}'}_j, \theta_k', j)\}} - \mathds{1}_{\{h(s^{\theta_{k}}_j, \theta_k, j) \neq \pi(s^{\theta_{k}}_j, \theta_k, j)\}}) \Big)\Big]$$ $$= \mathbb{E}_{S,S'}\Big[\sup_{h \in \mathcal{H}}\Big((\mathds{1}_{\{h(s^{\theta_{i'}'}_l, \theta_i', l) \neq \pi(s^{\theta_{i'}'}_l, \theta_i', l)\}} - \mathds{1}_{\{h(s^{\theta_{i}}_l, \theta_i, l) \neq \pi(s^{\theta_{i}}_l, \theta_i, l)\}})$$ $$+ \sum_{k\neq i}\sum_{j \neq l}(\mathds{1}_{\{h(s^{\theta_{k'}'}_j, \theta_k', j) \neq \pi(s^{\theta_{k'}'}_j, \theta_k', j)\}} - \mathds{1}_{\{h(s^{\theta_{k}}_j, \theta_k, j) \neq \pi(s^{\theta_{k}}_j, \theta_k, j)\}}) \Big)\Big],$$ which, repeating for all $i,l$, indicates that 
$$\mathbb{E}_{S,S'}\Big[\sup_{h \in \mathcal{H}}\Big(\sum_{i=1}^m\sum_{l=0}^{L-1}(\mathds{1}_{\{h(s^{\theta_{k'}'}_l, \theta_k', l) \neq \pi(s^{\theta_{k'}'}_l, \theta_k', l)\}} - \mathds{1}_{\{h(s^{\theta_{k}}_l, \theta_k, l) \neq \pi(s^{\theta_{k}}_l, \theta_k, l)\}}) \Big)\Big]$$ $$= \mathbb{E}_{S,S',\mathbf{\sigma}}\Big[\sup_{h \in \mathcal{H}}\Big(\sum_{i=1}^m\sum_{l=0}^{L-1}\sigma_{i,l}(\mathds{1}_{\{h(s^{\theta_{k'}'}_l, \theta_k', l) \neq \pi(s^{\theta_{k'}'}_l, \theta_k', l)\}} - \mathds{1}_{\{h(s^{\theta_{k}}_l, \theta_k, l) \neq \pi(s^{\theta_{k}}_l, \theta_k, l)\}}) \Big)\Big],$$ where $\mathbf{\sigma}$ denotes the $m\times L$ matrix of iid random signs. Noting that $$\sup_{h \in \mathcal{H}}\Big(\sum_{i=1}^m\sum_{l=0}^{L-1}\sigma_{i,l}(\mathds{1}_{\{h(s^{\theta_{k'}'}_l, \theta_k', l) \neq \pi(s^{\theta_{k'}'}_l, \theta_k', l)\}} - \mathds{1}_{\{h(s^{\theta_{k}}_l, \theta_k, l) \neq \pi(s^{\theta_{k}}_l, \theta_k, l)\}}) \Big)$$ $$\leq \sup_{h \in \mathcal{H}}\sum_{i=1}^m\sum_{l=0}^{L-1}\sigma_{i,l}\mathds{1}_{\{h(s^{\theta_{k'}'}_l, \theta_k', l) \neq \pi(s^{\theta_{k'}'}_l, \theta_k', l)\}}  + \sup_{h \in \mathcal{H}}\sum_{i=1}^m\sum_{l=0}^{L-1}-\sigma_{i,j}\mathds{1}_{\{h(s^{\theta_{k}}_l, \theta_k, l) \neq \pi(s^{\theta_{k}}_l, \theta_k, l)\}},$$ we have that, since $\sigma \sim -\sigma$, $$\mathbb{E}_{S,S',\mathbf{\sigma}}\Big[\sup_{h \in \mathcal{H}}\Big(\sum_{i=1}^m\sum_{l=0}^{L-1}\sigma_{i,l}(\mathds{1}_{\{h(s^{\theta_{k'}'}_l, \theta_k', l) \neq \pi(s^{\theta_{k'}'}_l, \theta_k', l)\}} - \mathds{1}_{\{h(s^{\theta_{k}}_l, \theta_k, l) \neq \pi(s^{\theta_{k}}_l, \theta_k, l)\}}) \Big)\Big]$$ $$\leq \mathbb{E}_{S,S',\sigma}\left[\sup_{h \in \mathcal{H}}\sum_{i=1}^m\sum_{l=0}^{L-1}\sigma_{i,j}\mathds{1}_{\{h(s^{\theta_{k'}'}_l, \theta_k', l) \neq \pi(s^{\theta_{k'}'}_l, \theta_k', l)\}}  + \sup_{h \in \mathcal{H}}\sum_{i=1}^m\sum_{l=0}^{L-1}\sigma_{i,l}\mathds{1}_{\{h(s^{\theta_{k}}_l, \theta_k, l) \neq \pi(s^{\theta_{k}}_l, \theta_k, l)\}}\right]$$ $$= 2mL\mathbb{E}_S[R(\mathcal{F}\circ S)],$$ implying that $$\mathbb{E}_S\left[\sup_{h \in \mathcal{H}}\left(\mathcal{L}^{m,L}_{\mathcal{D}, \pi}(h) - \mathcal{L}^{m,L}_S(h)\right)\right] \leq 2\mathbb{E}_S R(\mathcal{F}\circ S).$$
\end{proof}

Now, the classical approach would now rely on McDiarmid's concentration inequality. However, we are unable to do so due to the dependent structure on our sample $S$. Instead we use an analogous concentration inequality suited for Markov processes. To handle this setting, we define the notion of mixing time in the case of time-inhomogeneous Markov chains and then state the key lemma from \citet{paulin2015concentration}.

\begin{definition}
Let $X_1, \ldots, X_N$ be a Markov chain on Polish state space $\Omega_1\times \cdots\times \Omega_N$. Let $\mathcal{S}(X_{i+t}|X_i = x)$ be the conditional distribution of $X_{i+t}$ given $X_i = x$. Define $$\overline{d}(t) = \max_{1 \leq i \leq N-t}\sup_{x,y \in \Omega_i} d_{TV}(\mathcal{S}(X_{i+t}|X_i=x), \mathcal{S}(X_{i+t}|X_i=y)) \textnormal{ and }\tau(\epsilon) = \min\{t \in \mathbb{N}: \overline{d}(t) \leq \epsilon\}.$$ Furthermore, define $$\tau_{\textnormal{min}} = \inf_{0 \leq \epsilon < 1} \tau(\epsilon)\cdot \left(\frac{2-\epsilon}{1-\epsilon}\right)^2.$$
\end{definition}
\noindent
Then, from \citet{paulin2015concentration}, we use the following lemma:
\begin{lemma}
Let $X_1, \ldots, X_N$ be a Markov chain on Polish state space $\Lambda = \Lambda_1\times\cdots\times\Lambda_N$. Suppose that $f: \Lambda \rightarrow \mathbb{R}$ satisfies, for all $x,y \in \Lambda$ that $f(x) - f(y) \leq \sum_{i=1}^Nc_i\mathds{1}[x_i \neq y_i],$ for some $c_i \in \mathbb{R}_+$. Then, for any $t \geq 0$, $$P\left(|f(x) - \mathbb{E}f(X)| \geq t\right) \leq 2\exp\left(\frac{-2t^2}{\tau_{\textnormal{min}}\sum_{i=1}^Nc_i^2}\right).$$
\end{lemma}

We rewrite the following immediate consequence of the claim of \citet{paulin2015concentration}, stated above, in the following form:

\begin{corollary}
Let $f: \mathcal{X}^n \rightarrow \mathbb{R}$ satisfy bounded differences. That is, for all $i$ and $\forall x_1, \ldots, x_n, x'_i \in \mathcal{X}$, $$|f(x_1, \ldots, x_n) - f(x_1, \ldots, x_{i-1}, x'_i, x_{i+1}, \ldots, x_n)| \leq c.$$ Then with probability at least $1-\delta$, $$|f(s_1, \ldots, s_n) - \mathbb{E}[f(s_1, \ldots, s_n)]| \leq c\sqrt{\frac{n\tau_{\textnormal{min}}\log(2/\delta)}{2}},$$ where again the states are sampled according to the Markov process induced by $\pi$.
\end{corollary}

Finally, we arrive at the main lemma, which we will use to prove the main result: a sample complexity bound in this special setting.
\begin{lemma}
Let $S$, containing $mL$ data points, be sampled according to the procedure described above. Then with probability at least $1-\delta, \forall h \in \mathcal{H}$, $$\mathcal{L}^{L}_{\mathcal{D}, \pi}(h) - \mathcal{L}^{m,L}_S(h) \leq 2\mathbb{E}_{S'}R(\mathcal{F}\circ S') + \sqrt{\frac{(3L+1)\tau_{\textnormal{min}}\log(2/\delta)}{2m}}.$$
\end{lemma}
\begin{proof}
Write $S$ as $S = (\theta_1, 0, s_0^{\theta_1}, a_0^{\theta_1}, \ldots, \theta_m, \ldots, L-1, s_{L-1}^{\theta_m}, a_{L-1}^{\theta_m})$, define $$\zeta(\theta_1, 0, s_0^{\theta_1}, a_0^{\theta_1}, \ldots, \theta_m, \ldots, L-1, s_{L-1}^{\theta_m}, a_{L-1}^{\theta_m}) = \sup_{h \in \mathcal{H}}(\mathcal{L}^{L}_{\mathcal{D}, \pi}(h) - \mathcal{L}^{m,L}_S(h)).$$ We verify that $\zeta$ satisfies the bounded differences condition above. Letting $S'$  be the vector equal to $S$ except in one component (which may be a component containing either a $\theta$ or a $s^\theta$ or a time index or $a^\theta$), we have that $$|\zeta(S) - \zeta(S')| = |\sup_{h \in \mathcal{H}}(\mathcal{L}^{L}_{\mathcal{D}, \pi}(h) - \mathcal{L}^{m,L}_S(h)) - \sup_{h \in \mathcal{H}}(\mathcal{L}^{L}_{\mathcal{D}, \pi}(h) - \mathcal{L}^{m,L}_{S'}(h))|$$ $$\leq \sup_{h \in \mathcal{H}}|\mathcal{L}^{m,L}_{S'}(h) - \mathcal{L}^{m,L}_S(h)| \leq 1/m,$$ since, the worst case scenario is when the entry which disagrees is a $\theta$ entry, which will result in, at worst, incorrect classification for each state in the trajectory associated with that $\theta$ resulting in $L$ mistakes. Therefore, by the above corollary, we get that, with probability at least $1-\delta$, $$\sup_{h \in \mathcal{H}}(\mathcal{L}^{m,L}_{\mathcal{D}, \pi}(h) - \mathcal{L}^{m,L}_S(h)) \leq \mathbb{E}\sup_{h \in \mathcal{H}}(\mathcal{L}^{m,L}_{\mathcal{D}, \pi}(h) - \mathcal{L}^{m,L}_S(h)) + \frac{1}{m}\sqrt{\frac{(m + 3mL)\tau_{\textnormal{min}}\log(2/\delta)}{2}},$$ which, by Lemma A.0.2 is at most, with probability at least $1-\delta$, $$2\mathbb{E}_{S'}R(\mathcal{F}\circ S') + \sqrt{\frac{(3L+1)\tau_{\textnormal{min}}\log(2/\delta)}{2m}},$$ implying the desired result.
\end{proof}

Now, we will upper bound $R(\mathcal{F}\circ S)$ for any dataset $S$. First, we recall the definition of Natarajan dimension, Natarajan's lemma, and recall that our hypothesis is a classifier on $|A|$ different classes:

\begin{definition}[Natarajan Dimension for $k$-class classification]
A subset $S \subset \mathcal{X}$ is said to be shattered by $\mathcal{H}$ if there exist functions $f_0, f_1: C \rightarrow [k]$ such that for all $B \subset S$, there exists $h \in \mathcal{H}$ such that $h$ agrees with $f_0$ on $B$ and agrees with $f_1$ on $S\backslash B$.
\end{definition}

\begin{lemma}[Natarajan's]
Let $\mathcal{H}$ be a hypothesis class from some finite set $\mathcal{X} \rightarrow [k]$. Then $$|\mathcal{H}| \leq |\mathcal{X}|^{\textnormal{Ndim}(\mathcal{H})}\cdot k^{2\textnormal{Ndim}(\mathcal{H})}.$$
\end{lemma}

From this lemma, we see that $$\left|\left\{\left(\mathds{1}_{\{h(s_0^{\theta_1}, \theta_1, 0) \neq \pi(s_0^{\theta_1}, \theta_1, 0)\}}, \ldots, \mathds{1}_{\{h(s_{L-1}^{\theta_m}, \theta_m, L-1) \neq \pi(s_{L-1}^{\theta_m}, \theta_m,L-1)\}}\right): h \in \mathcal{H}\right\}\right| \leq (mL)^{\textnormal{Ndim}(\mathcal{H})} \cdot |A|^{2\textnormal{Ndim}(\mathcal{H})}.$$

Now, we state Massart's Lemma:
\begin{lemma}[Massart]
Let $B = \{b_1, \ldots, b_N\}$ denote a finite set of vectors in $\mathbb{R}^m$. Define $\overline{b} = \frac{1}{N}\sum_{i=1}^Nb_i$. Then $$R(B) \leq \max_{b \in B}||b - \overline{b}||\frac{\sqrt{2\log(N)}}{m}.$$
\end{lemma}

Thus, the above two lemmas imply that, letting $d = \textnormal{Ndim}(\mathcal{H})$, $$R(\mathcal{F}\circ S') \leq \sqrt{\frac{2d\left(\log(mL) + 2\log(|A|)\right)}{mL}}.$$ Thus, by a union-bound we have that with probability at least $1-\delta$, for every $h \in \mathcal{H}$, we have $$|\mathcal{L}^{m,L}_{\mathcal{D}, \pi}(h) - \mathcal{L}^{m,L}_S(h)| \leq 2\sqrt{\frac{2d\left(\log(mL) + 2\log(|A|)\right)}{mL}} + \sqrt{\frac{(3L+1)\tau_{\textnormal{min}}\log(4/\delta)}{2m}}.$$ So, in order to ensure that this is at most $\epsilon$, we require that $$2\sqrt{\frac{2d\left(\log(mL) + 2\log(|A|)\right)}{mL}} + \sqrt{\frac{(3L+1)\tau_{\textnormal{min}}\log(4/\delta)}{2m}} \leq \epsilon$$ $$\impliedby 2\sqrt{\frac{8d\left(\log(mL) + 2\log(|A|)\right)}{mL}+ \frac{(3L+1)\tau_{\textnormal{min}}\log(4/\delta)}{2m}} \leq \epsilon$$ $$\impliedby 2\sqrt{\frac{8d\left(\log(mL) + 2\log(|A|)\right)/l + (3L+3)\tau_{\textnormal{min}}\log(4/\delta)}{m}} \leq \epsilon$$ $$\impliedby m \geq \frac{4}{\epsilon^2}\left(\frac{8d}{L}\left(\log(m) + \log(L) + 2\log(|A|)\right) + (3L+3)\tau_{\textnormal{min}}\log(4/\delta)\right),$$ which, by the below lemma (whose proof can be found in \citet{shalev2014understanding}), is implied by $$m \geq \frac{128d}{\epsilon^2 L}\log\left(\frac{64d}{\epsilon^2 L}\right) + \frac{8}{\epsilon^2}\left(\frac{8d}{L}(\log(L) + 2\log(|A|)) + (3L+3)\tau_{\textnormal{min}}\log(4/\delta)\right).$$

\begin{lemma}
Let $a \geq 1, b> 0$. Then $x \geq 4a\log(2a) + 2b \implies x \geq a \log(x) + b$.
\end{lemma}

Replacing the $\epsilon$ above with $\epsilon/2$ gives the desired sample complexity upper bound:
\begin{lemma}
Let $\mathcal{A}$ denote an ERM learning algorithm. Then for any distribution over the data, if $\mathcal{A}$ receives $$m = O\left(\frac{d}{\epsilon^2}\left(\log\left(\frac{d}{\epsilon L}\right) + \frac{d}{L}(\log(L) + \log(|A|)) + L\tau_{\textnormal{min}}\log(1/\delta)\right)\right),$$ then $\mathcal{A}$ returns a hypothesis in $\mathcal{H}$ which has error at most $\epsilon$ greater than the true error minimizer in $\mathcal{H}$ with probability at least $1-\delta$.
\end{lemma}

Finally, we return to the definition of $\tau_{\textnormal{min}}$:
\begin{definition}
Let $X_1, \ldots, X_N$ be a Markov chain on Polish state space $\Omega_1\times \cdots\times \Omega_N$. Let $\mathcal{L}(X_{i+t}|X_i = x)$ be the conditional distribution of $X_{i+t}$ given $X_i = x$. Define $$\overline{d}(t) = \max_{1 \leq i \leq N-t}\sup_{x,y \in \Omega_i} d_{TV}(\mathcal{L}(X_{i+t}|X_i=x), \mathcal{L}(X_{i+t}|X_i=y)) \textnormal{ and }\tau(\epsilon) = \min\{t \in \mathbb{N}: \overline{d}(t) \leq \epsilon\}.$$ Furthermore, define $$\tau_{\textnormal{min}} = \inf_{0 \leq \epsilon < 1} \tau(\epsilon)\cdot \left(\frac{2-\epsilon}{1-\epsilon}\right)^2.$$
\end{definition}

Notice that, in our case, looking at the proof of Lemma A.0.4, $t(\epsilon) \leq L, \forall 0 \leq \epsilon \leq 2L$ (since the entire process 'restarts' with a new draw of $\theta$), and so we have that $\tau_{\textnormal{min}} \leq 8L$, thus giving us the desired theorem on sample complexity:
\begin{theorem}
Let the concept class $\mathcal{H}$ have Natarajan dimension $d$. There exists a learning algorithm $\mathcal{A}$ such that for any distribution over the data, there exists $m$ with $$m = O\left(\frac{d}{\epsilon^2}\left(\log\left(\frac{d}{\epsilon L}\right) + \frac{d}{L}(\log(L) + \log(|A|)) + L^2\log(1/\delta)\right)\right),$$ such that if $\mathcal{A}$ receives at least $m$ $L$-long trajectories in the batch, then $\mathcal{A}$ returns a hypothesis in $\mathcal{H}$ which has classification error at most $\epsilon$ greater than the true error minimizer of $\mathcal{H}$ with probability at least $1-\delta$.
\end{theorem}

\subsection{Hard Family of CMDPs}

\subsubsection{An $\Omega(\frac{|\mathcal{H}|}{\epsilon^2})$ Lower Bound}
We consider the following $K$-armed contextual bandit problem. The hypothesis class will be a set of functions $f: \{0, \ldots, L-1\} \rightarrow [K]$. The $i$th context vectors will simply be $i$ for $i = 0, \ldots, L-1$; that is, the only information the context gives is the round number. $f^*(i)$ denotes the optimal arm for round $i$. For the first round, the optimal arm gives $\textnormal{Bern}(\frac{1}{2} + \frac{3}{2}\epsilon)$ reward and all other arms give $\textnormal{Bern}(1/2)$ rewards. For all other contexts, all arms give $\textnormal{Bern}(1/2)$ rewards. With this structure, we further restrict the hypothesis class so that for any $f \neq f' \in \mathcal{H}$, $f(i) = f'(i), \forall \theta \in \Theta, i = 1, \ldots, L-1$; that is, all hypotheses agree on all but the very first round. On the first round, the set $\{f(0): f \in \mathcal{H}\} = A$. Therefore, we simply take $|\mathcal{H}| = |A|$. Notice that for each round, the final $L-1$ values seen provide no information. Thus, only the first context is important, and the problem reduces to the standard $K$-arm bandit problem of \citet{krishnamurthy2016pac}, and thus we get the lower bound $\Omega\left(\frac{K}{\epsilon^2}\right)$, as desired. Thus, we have shown that we can remove the factor of $\frac{1}{L}$ from the immediate application of the result of \citet{krishnamurthy2016pac}.

\subsubsection{Linear Growth in $|\Theta|$ with Small Hypothesis Class}
The standard construction of \citet{chen2019information} as well as \citet{krishnamurthy2016pac} consider the problem for single MDPs and POMDPs, respectively, and note that it is information-theoretically equivalent to the best-arm identification identification problem, with number of arms equal to the size of the model class. Unfortunately, in our setting, this is not the case, however we are still able to perform an information-theoretic analysis which is very similar to that of \citet{krishnamurthy2016pac} as well as \citet{10.1137/S0097539701398375}. 

First, we consider the following family of $K$-arm bandit problems. Let $S_K$ denote the set of permutations on $[K]$, and let $\Gamma$ denote a maximal subset of $S_K$ such that any two permutations in $\Gamma$ disagree on all inputs. It is clear that $|\Gamma| \leq K$, and that equality can be achieved as follows: using the standard group theoretic notation, we simply let $\Gamma = \langle (1 \, 2 \, \cdots \, K)\rangle$; that is, the subgroup of $S_K$ generated by the permutation taking $i\mapsto i+1$ for $i = 1, \ldots, K-1$, and taking $K \mapsto 1$. It is clear that $|\Gamma| = K$ and that $$\forall \sigma \neq \rho \in \Gamma, \sigma(i) \neq \rho(i), \forall i \in [K].$$ For $\sigma \in \Gamma$, let $\mathcal{M}_\sigma$ denote a family of $K$ bandit problems, each indexed by $\theta \in [K]$; specifically, $\mathcal{M}^\theta_\sigma$ is a $K$-arm bandit with optimal arm $\sigma(\theta)$, giving reward from $\textnormal{Bern}\left(\frac{1}{2} + \frac{3}{2}\epsilon\right)$ and all other arms giving reward from $\textnormal{Bern}(1/2)$. The hypothesis class of bandit families is thus $$\mathcal{H} = \{\mathcal{M}_\sigma|\sigma \in \Gamma\}.$$ When the target family is $\mathcal{M}_\sigma$, at each stage, a value of $\theta$ is drawn i.i.d from $\textnormal{Unif}([K])$, and the agent is allowed to interact with the bandit $\mathcal{M}^\theta_\sigma$ for $L$ iterations. We would like to lower bound $m$, the number of draws of $\theta$ required for any algorithm to choose an $\epsilon$-optimal permutation (where, in this case, the value of the permutation $\sigma$ is evaluated the same sense as in the standard CMDP setting: by taking $\mathbb{E}_\theta[R(\sigma(\theta))]$) with confidence $\delta$. Notice that, in this setting, the only $\epsilon$-optimal permutation is the optimal one: suppose that $\sigma^*$ is the optimal permutation, and $\sigma$ is some other permutation in $\Gamma$. Then since $\sigma$ and $\sigma^*$ disagree on all values in $[K]$, we will have that the value of $\sigma$ will be $\frac{3}{2}\epsilon > \epsilon$ less than $\sigma^*$, since, for each value of $\theta$, $\sigma$ selects a suboptimal arm, thus implying that $\sigma^*$ is indeed the only $\epsilon$-optimal permutation.

To achieve a lower bound, we represent a deterministic algorithm as sequence of mappings $f_t: (\{0,1\}\times\Theta)^t \rightarrow \Gamma$ for $0 \leq t \leq mL$, where we will interpret $f_t(r^t)$ as the permutation the algorithm believes is optimal after seeing the reward sequence $r^t \in (\{0,1\}\times \Theta)^t$ (our reward sequences are 'labelled' by the value of $\theta$ which they correspond to; that is, the reward sequence is a $t$-long sequence of ordered pairs, where the $i$th pair is the pair containing the reward at time $i$ and the value of $\theta$ at time $i$). In particular, this belief satisfies the property that for $(k-1)L + 1 \leq t \leq kL$, $\left(f_t(r^t)\right)(\theta_k)$ is the action taken by the algorithm after seeing reward sequence $r^t$ (where here the $m$ values of $\theta$ drawn are $\theta_1, \ldots, \theta_m$). The return of the algorithm after seeing the reward sequence $r^{mL}$ is thus the permutation $f_{mL}(r^{mL})$.

Following the notation of \citet{krishnamurthy2016pac} we will then let $\mathbb{P}_{\sigma^*, f}$ denote the distribution over all $mL$ rewards when $\sigma^*$ is the optimal permutation and the algorithm and actions selected are done so according to $f$. We let $\mathbb{P}_{0,f}$ denote the same distribution as above, except in the setting when all permutations give the same family of bandits whose rewards are always drawn from $\textnormal{Bern}(1/2)$. 

Now, we have the following $$|\mathbb{P}_{\sigma^*, f}(f_{mL} = \sigma^*) - \mathbb{P}_{0, f}(f_{mL} = \sigma^*)| \leq |\mathbb{P}_{\sigma^*, f} - \mathbb{P}_{0,f}|_{TV} \leq \sqrt{\frac{1}{2}KL(\mathbb{P}_{0, f} || \mathbb{P}_{\sigma^*, f})},$$ where the last inequality is Pinsker's. From the definition of $KL$ divergence, we then have that $$KL(\mathbb{P}_{0, f} || \mathbb{P}_{\sigma^*, f}) = \sum_{r^{mL} \in \{0,1\}^{mL}} \mathbb{P}_{0,f}(r^{mL})\log\left(\frac{\mathbb{P}_{0,f}(r^{mL})}{\mathbb{P}_{\sigma^*,f}(r^{mL})}\right),$$ which, by the chain rule for $KL$ divergence is $$\sum_{t=1}^{mL}\sum_{r^t \in \{0,1\}^t} \mathbb{P}_{0,f}(r^t)\log\left(\frac{\mathbb{P}_{0,f}(r^{t}|r^{t-1})}{\mathbb{P}_{\sigma^*,f}(r^{t}|r^{t-1})}\right),$$ which, by noticing that if, at time $t$, any suboptimal action is selected, the log ratio will be zero, letting $$\Lambda_{t-1} = \{r^{t-1} \in \{0,1\}^{t-1}: \left(f_{t-1}(r^{t-1})\right)(\theta_{\lfloor\frac{t-2}{L}\rfloor + 1}) = \sigma^*(\theta_{\lfloor\frac{t-2}{L}\rfloor + 1})\},$$ the above is simply $$\sum_{t=1}^{mL}\sum_{r^{t-1} \in \Lambda_{t-1}}\mathbb{P}_{0,f}(r^{t-1})\left(\sum_{x \in \{0,1\}}\mathbb{P}_{0,f}(x)\log\left(\frac{\mathbb{P}_{0,f}(r^{t}|a_t \textnormal{ is optimal})}{\mathbb{P}_{\sigma^*,f}(r^{t}|a_t \textnormal{ is optimal})}\right)\right)$$ $$= \sum_{t=1}^{mL}\sum_{r^{t-1} \in \Lambda_{t-1}}\mathbb{P}_{0,f}(r^{t-1})\left(\frac{1}{2}\log\left(\frac{1/2}{1/2 - \frac{3}{2}\epsilon}\right) + \frac{1}{2}\log\left(\frac{1/2}{1/2 + \frac{3}{2}\epsilon}\right)\right)$$ $$= \left(-\frac{1}{2}\log\left(1 - 9\epsilon^2\right)\right)\sum_{t=1}^{mL}\sum_{r^{t-1} \in \Lambda_{t-1}}\mathbb{P}_{0,f}(r^{t-1})$$ $$= \left(-\frac{1}{2}\log\left(1 - 9\epsilon^2\right)\right)\sum_{t=1}^{mL}\mathbb{P}_{0,f}[\left(f_{t-1}(r^{t-1})\right)(\theta_{\lfloor\frac{t-2}{L}\rfloor + 1}) = \sigma^*(\theta_{\lfloor\frac{t-2}{L}\rfloor + 1})],$$ where in the above we have slightly abused notation, and any reference to $\Lambda_{t-1}$ or $f_{t-1}(r^{t-1})$ for $t = 1$ refers to the initial action that the algorithm takes. Thus, again with the notation of \citet{krishnamurthy2016pac}, letting $N_{\sigma^*}$ denote the random variable equal to the number of times the algorithm chose the arm given by $\sigma^*$ in the bandit family where all arms are always equal, we have that $$\mathbb{P}_{\sigma^*, f}(f_{mL} = \sigma^*) - \mathbb{P}_{0,f}(f_{mL} = \sigma^*) \leq \frac{1}{2}\sqrt{-\mathbb{E}_{0,f}[N_{\sigma^*}]\log\left(1 - 9\epsilon^2\right)}.$$ Taking the expectation over all possible optimal permutations then gives that $$\frac{1}{|\Gamma|}\sum_{\sigma^* \in \Gamma}\mathbb{P}_{\sigma^*, f}(f_{mL} = \sigma^*) \leq \frac{1}{|\Gamma|}\sum_{\sigma^* \in \Gamma}\mathbb{P}_{0, f}(f_{mL} = \sigma^*) + \frac{1}{2|\Gamma|}\sum_{\sigma^* \in \Gamma}\sqrt{-\mathbb{E}_{0,f}[N_{\sigma^*}]\log\left(1 - 9\epsilon^2\right)},$$ which, by Jensen's inequality is at most $$\frac{1}{|\Gamma|}\sum_{\sigma^* \in \Gamma}\mathbb{P}_{0, f}(f_{mL} = \sigma^*) + \frac{1}{2}\sqrt{-\frac{\log\left(1 - 9\epsilon^2\right)}{|\Gamma|}\sum_{\sigma^* \in \Gamma}\mathbb{E}_{0,f}[N_{\sigma^*}]}.$$ Now, notice that we can write $$\sum_{\sigma^* \in \Gamma}N_{\sigma^*} \leq \sum_{t=1}^{mL}|\{\sigma \in \Gamma: \left(f_{t-1}(r^{t-1})\right)(\theta_{\lfloor\frac{t-2}{L}\rfloor + 1}) = \sigma(\theta_{\lfloor\frac{t-2}{L}\rfloor + 1})\}|,$$ which, based on how we chose $\Gamma$, is at most $mL$. So, the above is at most $$\frac{1}{|\Gamma|} + \frac{1}{2}\sqrt{-\frac{\log\left(1 - 9\epsilon^2\right)mL}{|\Gamma|}}.$$ Noticing that $-\log(1-x) \leq 2x$ when $x \leq 1/2$, if $9\epsilon^2K^2 \leq \frac{1}{2}$, then the above is at most $$\frac{1}{|\Gamma|} + \frac{1}{2}\sqrt{\frac{18\epsilon^2mL}{|\Gamma|}}.$$ Thus, if we take $$m < \frac{|\Gamma|}{162\epsilon^2 L} = \frac{K}{162\epsilon^2 L},$$ the above will be at most $2/3$ since certainly $|\Gamma| \geq 2$, thus indicating that if $m$ is less than the above, we cannot achieve suboptimality of at most $\epsilon$ for any $\delta < 2/3$. 

The construction for achieving the $\Omega\left(\frac{|A|^L}{L\epsilon^2}\right)$ is then achieved in a similar way in \citet{krishnamurthy2016pac}. This can be done by describing the MDP for the context $\theta$ as a tree with special edge given by $\sigma^*(\theta)$. The hypothesis class is then $\langle(1 \, 2 \, \cdots \, |A|^{L-1})\rangle \subset S_{|A|^{L-1}}$, and $\Theta = [|A|^{L-1}]$. Thus, our lower bound grows at the rate $\Omega\left(\frac{|\Theta|}{L\epsilon^2}\right)$ while having a hypothesis class which has size $|\Theta|$.

\subsection{FQI Analysis}
First we give pseudocode describing FQI for CMDPs:
\begin{algorithm}[H]
\caption{CMDP-FQI}
\begin{algorithmic}
\STATE \textbf{Input: } Dataset $\mathcal{D}=\{(\theta_i, s^{\theta_i}_l, a^{\theta_i}_l, r^{\theta_i}_l, s^{\theta_{i'}}_{l})\}_{i=1, l=0}^{i=m, l=L-1}$
\STATE{Initialize $Q: S\times A\times \Theta\times\{0, \ldots, L\} \rightarrow \mathbb{R}$ with $Q(s,a,\theta,l) = 0, \forall s,a,\theta,l$}
\FOR{$l = 1, \ldots, L$}
    \STATE $\mathcal{D}_l \gets \emptyset$
    \FOR{$(\theta_i, s^{\theta_i}_l, a^{\theta_i}_l, r^{\theta_i}_l, s^{\theta_{i'}}_{l}) \in \mathcal{D}$}
        \STATE $\textnormal{input} \gets (s^{\theta_i}_l, a^{\theta_i}_l, \theta_i, l)$
        \STATE $\textnormal{target} \gets r^{\theta_i}_l + \max_{a \in A}Q(s^{\theta_{i'}}_{l}, a, \theta_i, l-1)$
        \STATE $\mathcal{D}_l \gets \mathcal{D}_l \cup \{(\textnormal{input}, \textnormal{target})\}$
    \ENDFOR
    \STATE $Q_l(\cdot,\cdot,\cdot,l) = \texttt{REGRESS}(\mathcal{D}_l)$
\ENDFOR
\RETURN $Q$
\end{algorithmic}
\end{algorithm}

In the above, the program $\texttt{REGRESS}(\mathcal{D}_l)$ returns the $Q$ function in $\mathcal{F}$ which, for the $l$th timestep, minimizes the empirical $\ell_2$ norm of the dataset. We use the notation of subscripting the $Q$ function by $l$ to indicate that the $l+1$th $Q$ function need not be the exact Bellman backup of the $l$th.

Now, our setup and derivation follows the same structure as \citet{chen2019information}. First, just as in \citet{chen2019information}, we define a semi-norm on real-valued functions with state-action inputs and also write FQI in terms of the backups $\mathcal{T}_l$. 

\begin{definition}
Define a semi-norm $||\cdot||_{p,\nu\times\mathcal{P}_\Theta}$ on functions $f: S\times A\times \Theta \rightarrow \mathbb{R}$ by $$||f||_{p,\nu\times \mathcal{P}_\Theta} = \left(\mathbb{E}_{\theta \sim \mathcal{P}_{\Theta}, (s,a)\sim\nu(\theta)}[f(s,a,\theta)^{1/p}]\right)^p$$
\end{definition}

Now, we can write the FQI above in terms of backups. For $l = 0, \ldots, L-1$, we say that $f(\cdot,\cdot,\cdot,l+1) = \widehat{\mathcal{T}_l^{\mathcal{F}}}f(\cdot,\cdot,\cdot,l)$, where, for $f' \in \mathcal{F}$, $$\widehat{\mathcal{T}^{\mathcal{F}}_l}f' = \textnormal{argmin}_{f \in \mathcal{F}}\mathcal{L}_l(f;f'),$$ and $$\mathcal{L}_l(f;f') = \frac{1}{|\mathcal{D}|}\sum_{(\theta, s^{\theta}, a^{\theta}_l, r^{\theta}, s^{\theta_{'}}) \in \mathcal{D}}(f(s^\theta,a^\theta,\theta,l+1) - r^\theta - V_{f'}(s^{\theta_{'}},\theta, l))^2.$$

We now state and prove the lemmas needed for the sample complexity upper bound as in \citet{chen2019information}.

\begin{lemma}
Let $\mu$ be the data distribution and $\nu$ an admissible distribution. Then $$||\cdot||_{2,\nu\times\mathcal{P}_\Theta} \leq \sqrt{C}||\cdot||_{2,\mu\times\mathcal{P}_\Theta},$$ where $\nu\times \mathcal{P}_\Theta$ is the distribution over $(\theta,s,a)$ triples with $\theta \sim \mathcal{P}_\Theta$, and $(s,a) \sim \nu(\theta)$, and $\mu\times\Theta$ is the distribution over $(\theta,s,a)$ triples with $\theta \sim \mathcal{P}_\Theta$, and $(s,a)\sim \mu$.
\end{lemma}
\begin{proof}
$$||f||_{2,\nu\times\mathcal{P}_\Theta} = \sqrt{\sum_{\theta \in \Theta}\mathcal{P}_\Theta (\theta)\sum_{(s,a) \in S\times A} (\nu(\theta))(s,a) f(s,a)^2}$$ $$\leq \sqrt{C\sum_{\theta \in \Theta}\mathcal{P}_\Theta (\theta)\sum_{(s,a) \in S\times A} \mu(s,a) f(s,a)^2} = \sqrt{C}||f||_{2,\mu\times\mathcal{P}_\Theta}$$
\end{proof}

We now rederive a standard infinite-horizon lemma from \citet{kakade2002approximately} for our finite-horizon setting. 

\begin{lemma}
Let $\pi^*$ denote an optimal possibly time-dependent deterministic policy and let $\hat{\pi}$ be some time-dependent deterministic policy. Letting $\mathcal{M}^\pi_l(s, \theta)$ denote the marginal distribution of the $l$th state, given $\theta$ and the initial state $s$ by following policy $\pi$, then $$V^{L}_{\pi^*}(s;\theta) - V^{L}_{\hat{\pi}}(s;\theta)= \sum_{l=0}^{L-1}\mathbb{E}_{s_l \sim \mathcal{M}^{\hat{\pi}}_l(s,\theta)}[V^{L-l}_{\pi^*}(s_l; \theta) - Q^*(s_l, \hat{\pi}, \theta, L-l)].$$
\end{lemma}

\begin{proof}
$$V_L^{\hat{\pi}}(s;\theta) = \sum_{l=0}^{L-1}\mathbb{E}_{s_l \sim \mathcal{M}^{\hat{\pi}}_l(s,\theta)}[R(s_l, \hat{\pi}(s_l, \theta, l)]$$ $$= \sum_{l=0}^{L-1}\mathbb{E}_{s_l \sim \mathcal{M}^{\hat{\pi}}_l(s,\theta)}[R(s_l, \hat{\pi}(s_l, \theta, l)) + V^{L-l}_{\pi^*}(s_l;\theta) - V^{L-l}_{\pi^*}(s_l;\theta)]$$ $$= \sum_{l=0}^{L-2}\mathbb{E}_{s_l \sim \mathcal{M}^{\hat{\pi}}_l(s,\theta), s_{l+1} \sim \mathcal{M}^{\hat{\pi}}_{l+1}(s,\theta)}[R(s_l, \hat{\pi}(s_l, \theta, l)) + V^{L-l-1}_{\pi^*}(s_{l+1};\theta) - V^{L-l}_{\pi^*}(s_l;\theta)]$$ $$+ V^L_{\pi^*}(s_0;\theta) + \mathbb{E}_{s_{L-1} \sim \mathcal{M}^{\hat{\pi}}_{L-1}(s,\theta)}[R(s_{L-1}, \hat{\pi}(s_{L-1}, \theta, L-1)) - V^1_{\pi^*}(s_{L-1};\theta)]$$ $$\implies V_L^{\hat{\pi}}(s;\theta) - V^L_{\pi^*}(s_0;\theta) = \sum_{l=0}^{L-1}\mathbb{E}_{s_l \sim \mathcal{M}^{\hat{\pi}}_l(s,\theta)}[Q^*(s_l, \hat{\pi}, \theta, L-l) - V^{L-l}_{\pi^*}(s_l;\theta)],$$ as desired.
\end{proof}

With this, we now show how to control the difference in value of two policies in terms of the semi-norm of the difference of their $Q$-functions.

\begin{lemma}
Let $f: S\times A\times \Theta \times \{0, \ldots, L\}$ and let $\hat{\pi} = \pi_{f}$ denote the policy, which at time-step $t$ is greedy with respect to $f(\cdot, \cdot, \theta, t)$ for all $t$. Then $$v^L_{\pi^*} - v^L_{\hat{\pi}} \leq \sum_{l=0}^{L-1}||Q^*(\cdot, \cdot, \cdot, L-l) - f(\cdot,\cdot, \cdot,L-l)||_{2, \mathcal{M}^{\hat{\pi}}_l\times \pi^*\times\mathcal{P}_\Theta} + ||Q^*(\cdot, \cdot, \cdot, L-l) - f(\cdot, \cdot, \cdot, L-l)||_{2, \mathcal{M}^{\hat{\pi}}_l\times\hat{\pi}\times\mathcal{P}_\Theta},$$ where the notation $\mathcal{M}^\pi_l\times \pi'\times\mathcal{P}_\Theta$ denote the distribution over state action context triples in which $\theta \sim \mathcal{P}_\Theta$, then $l$th state, $s$, is drawn from $\mathcal{M}^\pi_l(\theta)$ (i.e. the marginal distribution conditioned just on $\theta$) and, $a = \pi'(s, l)$.
\end{lemma}
\begin{proof}
From the above, we have that $$v^L_{\pi^*} - v^L_{\hat{\pi}} = \sum_{l=0}^{L-1}\mathbb{E}_{\theta \sim \mathcal{P}_\Theta, s \sim \mathcal{P}_{s_0}, s_l \sim \mathcal{M}^{\hat{\pi}}_l(s,\theta)}[V^{L-l}_{\pi^*}(s_l;\theta) - Q^*(s_l, \hat{\pi}, \theta, L-l)]$$ $$\leq \sum_{l=0}^{L-1}\mathbb{E}_{\theta \sim \mathcal{P}_\Theta, s \sim \mathcal{P}_{s_0}, s_l \sim \mathcal{M}^{\hat{\pi}}_l(s,\theta)}[V^{L-l}_{\pi^*}(s_l;\theta) - f(s,\pi^*, \theta,L-l) + f(s,\hat{\pi}, \theta,L-l) - Q^*(s_l, \hat{\pi}, \theta, L-l)]$$ $$\leq \sum_{l=0}^{L-1}||Q^*(\cdot, \cdot, \cdot, L-l) - f(\cdot,\cdot, \cdot,L-l)||_{1, \mathcal{M}^{\hat{\pi}}_l\times \pi^*\times\mathcal{P}_\Theta} + ||Q^*(\cdot, \cdot, \cdot, L-l) - f(\cdot, \cdot, \cdot, L-l)||_{1, \mathcal{M}^{\hat{\pi}}_l\times\hat{\pi}\times\mathcal{P}_\Theta}$$ $$\leq \sum_{l=0}^{L-1}||Q^*(\cdot, \cdot, \cdot, L-l) - f(\cdot,\cdot, \cdot,L-l)||_{2, \mathcal{M}^{\hat{\pi}}_l\times \pi^*\times\mathcal{P}_\Theta} + ||Q^*(\cdot, \cdot, \cdot, L-l) - f(\cdot, \cdot, \cdot, L-l)||_{2, \mathcal{M}^{\hat{\pi}}_l\times\hat{\pi}\times\mathcal{P}_\Theta},$$ as desired.
\end{proof}

\begin{lemma}
Let $f, f': S\times A\times \Theta \times \{0, \ldots, L\} \rightarrow \mathbb{R}$ and $\pi_{f,f'}(s,\theta, l) = \textnormal{argmax}_{a \in A}\max\{f(s,a,\theta,l), f'(s,a,\theta,l)\}$. Then, for all admissible distributions $\nu$ and all $l \in \{0, \ldots, L\}$, $$||V^l_f(\cdot; \cdot) - V^l_{f'}(\cdot; \cdot)||_{2, P(\nu)\times\mathcal{P}_\Theta} \leq ||f(\cdot,\cdot,\cdot,l) - f'(\cdot,\cdot,\cdot,l)||_{2, P(\nu)\times \pi_{f,f'}(\cdot,\cdot, l)\times \mathcal{P}_\Theta},$$ where $P(\nu)\times \mathcal{P}_\Theta$ denotes the distribution over context-state pairs $(s',\theta)$ given that $\theta \sim \mathcal{P}_\Theta$, $(s,a) \sim \nu(\theta)$, and $s' \sim T(\cdot|s,a,\theta)$, and $P(\nu)\times \pi_{f,f'}(\cdot,\cdot, l)\times \mathcal{P}_\Theta$ denotes the distribution over triples $(s',a,\theta)$ where $(s',\theta) \sim P(\nu)\times \mathcal{P}_\Theta$ and $a = \pi_{f,f'}(s,\theta, l)$.
\end{lemma}


\begin{proof}
We have that $$||V^l_f(\cdot; \cdot) - V^l_{f'}(\cdot; \cdot)||^2_{2, P(\nu)\times\mathcal{P}_\Theta}$$ $$= \sum_{\theta \in \Theta}\mathcal{P}_\Theta(\theta)\sum_{(s,a) \in S\times A}(\nu(\theta))(s,a)\sum_{s' \in S}T(s'|s,a,\theta)\left(\max_{a \in A}f(s',a,\theta,l) -\max_{a \in A}f'(s',a,\theta,l) \right)^2$$ $$\leq \sum_{\theta \in \Theta}\mathcal{P}_\Theta\sum_{(s,a) \in S\times A}(\nu(\theta))(s,a)\sum_{s' \in S}T(s'|s,a,\theta)(f(s',\pi_{f,f'}(s',\theta,l),\theta,l) - f'(s',\pi_{f,f'}(s',\theta,l),\theta,l))^2$$ $$= ||f(\cdot,\cdot,\cdot,l) - f'(\cdot,\cdot,\cdot,l)||^2_{2, P(\nu)\times \pi_{f,f'}(\cdot,\cdot, l)\times \mathcal{P}_\Theta},$$ as desired.
\end{proof}

We now upper bound $||f - Q^*||_{2,\nu\times\mathcal{P}_\Theta}$. 

\begin{lemma}
For any data distribution $\mu$ over $S\times A$ and admissible distribution $\nu$, policy (which is potentially time-dependent), $\pi$, and $f_{l+1},f_l: S\times A \times\Theta\times\{0, \ldots, L-1\}\rightarrow R$ gotten from FQI, then we have $$||f_{l+1}(\cdot, \cdot, \cdot, l+1) - Q^*(\cdot, \cdot, \cdot, l+1)||_{2,\nu\times \mathcal{P}_\Theta}$$ $$\leq \sqrt{C}||f_{l+1}(\cdot, \cdot, \cdot, l+1) - (\mathcal{T}_l(\cdot))f_l(\cdot, \cdot, \cdot, l)||_{2,\mu\times\mathcal{P}_\Theta} + ||f_l(\cdot,\cdot,\cdot,l) - Q^*(\cdot,\cdot,\cdot,l)||_{2, P(\nu)\times \pi_{f_l,Q^*}(\cdot,\cdot, l)\times \mathcal{P}_\Theta},$$ for all $l$.
\end{lemma}
\begin{proof}
We have $$||f_{l+1}(\cdot, \cdot, \cdot, l+1) - Q^*(\cdot, \cdot, \cdot, l+1)||_{2,\nu\times \mathcal{P}_\Theta}$$ $$= ||f_{l+1}(\cdot, \cdot, \cdot, l+1) - (\mathcal{T}_l(\cdot))f_l(\cdot,\cdot,\cdot,l) + (\mathcal{T}_l(\cdot))f_l(\cdot,\cdot,\cdot,l) - Q^*(\cdot, \cdot, \cdot, l+1)||_{2,\nu\times \mathcal{P}_\Theta}$$ $$\leq ||f_{l+1}(\cdot, \cdot, \cdot, l+1) - (\mathcal{T}_l(\cdot))f_l(\cdot,\cdot,\cdot,l)||_{2,\nu\times \mathcal{P}_\Theta} + ||(\mathcal{T}_l(\cdot))f_l(\cdot,\cdot,\cdot,l) - (\mathcal{T}_l(\cdot))Q^*(\cdot, \cdot, \cdot, l)||_{2,\nu\times \mathcal{P}_\Theta}.$$ Now, notice that $$||(\mathcal{T}_l(\cdot))f_l(\cdot,\cdot,\cdot,l) - (\mathcal{T}_l(\cdot))Q^*(\cdot, \cdot, \cdot, l)||_{2,\nu\times \mathcal{P}_\Theta}^2$$ $$= \mathbb{E}_{\theta \sim \mathcal{P}_\Theta, (s,a)\sim \nu(\theta)}[\left((\mathcal{T}_l(\theta))f_l(s,a,\theta,l) - (\mathcal{T}_l(\theta))Q^*(s, a, \theta, l)\right)^2]$$ $$= \mathbb{E}_{\theta \sim \mathcal{P}_\Theta, (s,a)\sim \nu(\theta)}[\left(\mathbb{E}_{s' \sim T(\cdot|s,a,\theta)}[V^l_{f_l}(s';\theta) - V^l_{\pi^*}(s';\theta)]\right)^2]$$ $$\leq \mathbb{E}_{\theta \sim \mathcal{P}_\Theta, (s,a)\sim \nu(\theta), s' \sim T(\cdot|s,a,\theta)}[\left(V^l_{f_l}(s';\theta) - V^l_{\pi^*}(s';\theta)\right)^2] = ||V^l_{f_l}(\cdot;\cdot) - V^l_{\pi^*}(\cdot;\cdot)||_{2, P(\nu)\times\mathcal{P}_\Theta}^2.$$ Thus, the above is at most $$||f_{l+1}(\cdot, \cdot, \cdot, l+1) - (\mathcal{T}_l(\cdot))f_l(\cdot,\cdot,\cdot,l)||_{2,\nu\times \mathcal{P}_\Theta} + ||V^l_{f_l}(\cdot;\cdot) - V^l_{\pi^*}(\cdot;\cdot)||_{2, P(\nu)\times\mathcal{P}_\Theta}$$ $$\leq \sqrt{C}||f_{l+1}(\cdot, \cdot, \cdot, l+1) - (\mathcal{T}_l(\cdot))f_l(\cdot,\cdot,\cdot,l)||_{2,\mu\times \mathcal{P}_\Theta} + ||V_{f_l}(\cdot,\cdot,l) - V^l_{\pi^*}(\cdot;\cdot)||_{2, P(\nu)\times\mathcal{P}_\Theta},$$ which, by the previous lemma, is at most $$\sqrt{C}||f_{l+1}(\cdot, \cdot, \cdot, l+1) - (\mathcal{T}_l(\cdot))f_l(\cdot,\cdot,\cdot,l)||_{2,\mu\times \mathcal{P}_\Theta} + ||f_l(\cdot,\cdot,\cdot,l) - Q^*(\cdot, \cdot,\cdot, l)||_{2, P(\nu)\times \pi_{f_l,Q^*}(\cdot,\cdot, l)\times \mathcal{P}_\Theta},$$ as desired.
\end{proof}

We now assume access to a function class $\mathcal{G}$ which approximates Bellman backups of $\mathcal{F}$. In particular, we define the error of approximation as follows:

\begin{definition}
Define $\epsilon_{\mathcal{F}, \mathcal{G}}$ to be the smallest real number such that if we let $$g^*_f(l) = \textnormal{argmin}_{g \in \mathcal{G}}||g(\cdot, \cdot, \cdot, l+1) - (\mathcal{T}_l(\cdot))f(\cdot, \cdot, \cdot, l)||_{2, \mu\times\mathcal{P}_\Theta},$$ then $$||(g^*_f(l))(\cdot, \cdot, \cdot, l+1) - (\mathcal{T}_l(\cdot))f(\cdot, \cdot, \cdot, l)||^2_{2, \mu\times\mathcal{P}_\Theta} \leq \epsilon_{\mathcal{F}, \mathcal{G}}, \forall l.$$
\end{definition}

\begin{lemma}
Let the dataset $\mathcal{D}$ be generated as described in the text. Then, for any $l \in \{0, \ldots, L\}$ and $f \in \mathcal{F}$, we have that, with probability at least $1-\delta$, $$\mathbb{E}_{\mathcal{D}}[\mathcal{L}_l(\widehat{\mathcal{T}^{\mathcal{G}}}_lf; f) - \mathcal{L}_l(g^*_f(l); f)] \leq \frac{56L^2\ln\left(\frac{L|\mathcal{F}||\mathcal{G}|}{\delta}\right)}{3m} + \sqrt{\frac{32L^2\ln\left(\frac{L|\mathcal{F}||\mathcal{G}|}{\delta}\right)}{m}\epsilon_{\mathcal{F}, \mathcal{G}}}$$
\end{lemma}
\begin{proof}
Fix $l$, and just as in \citet{chen2019information}, define $$X(g,f,g^*_f(l)) = (g(s,a,\theta,l+1) - r - V_f(s', \theta, l))^2 - ((g^*_f(l))(s,a,\theta,l+1) - r - V_f(s', \theta, l))^2.$$ Letting $X_{i,l}$ denote the evaluation of $X$ on $(\theta_i, s^{\theta_i}_l, a^{\theta_i}_l, r^{\theta_i}_l, s^{\theta_{i'}}_{l})$ it is clear that $$\frac{1}{mL}\sum_{i=1}^m\sum_{l=0}^{L-1}X_{i,l}(g,f,g^*_f(l)) = \mathcal{L}_l(g;f)- \mathcal{L}_l(g^*_f(l); f).$$ Now, we have that $$\textnormal{Var}[X(g,f,g^*_f(l))] \leq \mathbb{E}[X(g,f,g^*_f(l))^2]$$ $$= \mathbb{E}[\left((g(s,a,\theta,l+1) - r - V_f(s', \theta, l))^2 - ((g^*_f(l))(s,a,\theta,l+1) - r - V_f(s', \theta, l))^2\right)^2]$$ $$= \mathbb{E}[(g(s,a,\theta,l+1) - (g^*_f(l))(s,a,\theta,l+1))^2(g(s,a,\theta,l+1) + (g^*_f(l))(s,a,\theta,l+1) - 2r - 2V_f(s',\theta,l))^2].$$ Now, since all rewards are in $[0,1]$, $$g(s,a,\theta,l+1) + (g^*_f(l))(s,a,\theta,l+1) - 2r - 2V_f(s',\theta,l) \leq 2l$$ and so the above is at most $$4l^2\mathbb{E}[(g(s,a,\theta,l+1) - (g^*_f(l))(s,a,\theta,l+1))^2] = 4l^2||g(\cdot, \cdot, \cdot, l+1) - (g^*_f(l))(\cdot, \cdot, \cdot, l+1)||^2_{2, \mu\times \mathcal{P}_\Theta}$$ $$\leq 8l^2\left(||g(\cdot, \cdot, \cdot, l+1) - (\mathcal{T}_l(\cdot))f(\cdot, \cdot, \cdot, l)||^2_{2, \mu\times \mathcal{P}_\Theta} + ||(\mathcal{T}_l(\cdot))f(\cdot, \cdot, \cdot, l) - (g^*_f(l))(\cdot, \cdot, \cdot, l+1)||^2_{2, \mu\times \mathcal{P}_\Theta}\right)$$ $$= 8l^2\Big(||g(\cdot, \cdot, \cdot, l+1) - (\mathcal{T}_l(\cdot))f(\cdot, \cdot, \cdot, l)||^2_{2, \mu\times \mathcal{P}_\Theta} - ||(\mathcal{T}_l(\cdot))f(\cdot, \cdot, \cdot, l) - (g^*_f(l))(\cdot, \cdot, \cdot, l+1)||^2_{2, \mu\times \mathcal{P}_\Theta}$$ $$ + 2||(\mathcal{T}_l(\cdot))f(\cdot, \cdot, \cdot, l) - (g^*_f(l))(\cdot, \cdot, \cdot, l+1)||^2_{2, \mu\times \mathcal{P}_\Theta}\Big)$$ $$= 8l^2\left(\mathbb{E}[X(g,f,g^*_f(l))] + 2||(\mathcal{T}_l(\cdot))f(\cdot, \cdot, \cdot, l) - (g^*_f(l))(\cdot, \cdot, \cdot, l+1)||^2_{2, \mu\times \mathcal{P}_\Theta}\right)$$ $$\leq 8l^2\left(\mathbb{E}[X(g,f,g^*_f(l))] + 2\epsilon_{\mathcal{F}, \mathcal{G}}\right).$$ Now, notice that the random variables $X_{1,l}, X_{2,l}, \ldots, X_{m,l}$ are i.i.d for fixed $l$. Thus, as in \citet{chen2019information}, we will use one-sided Bernstein's inequality along with a union-bound over $f \in \mathcal{F}, g \in \mathcal{G}$ and finally over $l = 1, \ldots, L$. In particular, as \citet{chen2019information} note, union-bounding over the first two sets says that, with probability at least $1-\delta$, every $f \in \mathcal{F}, g \in \mathcal{G}$, and fixed $l$ satisfy $$\mathbb{E}[X(g,f,g^*_f(l))] - \sum_{i=1}^mX_{i,l}(g,f,g^*_f(l)) \leq \sqrt{\frac{2\textnormal{Var}[X(g,f,g_f^*(l))]\ln\left(\frac{|\mathcal{F}||\mathcal{G}|}{\delta}\right)}{m}} + \frac{4L^2\ln\left(\frac{|\mathcal{F}||\mathcal{G}|}{\delta}\right)}{3m}$$
$$\leq \sqrt{\frac{16l^2\left(\mathbb{E}[X(g,f,g^*_f(l))] + 2\epsilon_{\mathcal{F}, \mathcal{G}}\right)\ln\left(\frac{|\mathcal{F}||\mathcal{G}|}{\delta}\right)}{m}} + \frac{4L^2\ln\left(\frac{|\mathcal{F}||\mathcal{G}|}{\delta}\right)}{3m}.$$

Now, notice that $$\frac{1}{m}\sum_{i=1}^mX_{i,l}(\widehat{\mathcal{T}}^\mathcal{G}_lf, f, g^*_f(l)) \leq \frac{1}{m}\sum_{i=1}^mX_{i,l}(g^*_f(l), f, g^*_f(l)) = 0,$$ and so we really have that, with probability at least $1-\delta$, every $f \in \mathcal{F}$ satisfies $$\mathbb{E}[X(\widehat{\mathcal{T}}^\mathcal{G}_lf,f,g^*_f(l))] \leq \sqrt{\frac{16l^2\left(\mathbb{E}[X(\widehat{\mathcal{T}}^\mathcal{G}_lf,f,g^*_f(l))] + 2\epsilon_{\mathcal{F}, \mathcal{G}}\right)\ln\left(\frac{|\mathcal{F}||\mathcal{G}|}{\delta}\right)}{m}} + \frac{4L^2\ln\left(\frac{|\mathcal{F}||\mathcal{G}|}{\delta}\right)}{3m},$$ which, by algebraic manipulations (see Lemma 16 of \citet{chen2019information}) implies $$\mathbb{E}[X(\widehat{\mathcal{T}}^\mathcal{G}_lf,f,g^*_f(l))] \leq \frac{56L^2\ln\left(\frac{|\mathcal{F}||\mathcal{G}|}{\delta}\right)}{3m} + \sqrt{\frac{32L^2\ln\left(\frac{|\mathcal{F}||\mathcal{G}|}{\delta}\right)}{m}\epsilon_{\mathcal{F}, \mathcal{G}}},$$ for the fixed $l$. Having union-bounded with respect to $l$ will yield the desired upper bound of $$\frac{56L^2\ln\left(\frac{L|\mathcal{F}||\mathcal{G}|}{\delta}\right)}{3m} + \sqrt{\frac{32L^2\ln\left(\frac{L|\mathcal{F}||\mathcal{G}|}{\delta}\right)}{m}\epsilon_{\mathcal{F}, \mathcal{G}}}.$$

\end{proof}

Finally, we derive a sample-complexity bound:

\begin{theorem}
We can bound the error of the hypothesis returned by FQI, with probability at least $1-\delta$ as $$v^L_{\pi^*} - v^L_{\pi_f}\leq L(L+1)\sqrt{C\left(\frac{56L^2\ln\left(\frac{L|\mathcal{F}|^2}{\delta}\right)}{3m} + \sqrt{\frac{32L^2\ln\left(\frac{L|\mathcal{F}|^2}{\delta}\right)}{m}\epsilon_{\mathcal{F}, \mathcal{F}}} + \epsilon_{\mathcal{F}, \mathcal{F}}\right)}.$$ 
\end{theorem}

\begin{proof}
From Lemma A.1.5 we have $$||f_{l+1}(\cdot, \cdot, \cdot, l+1) - Q^*(\cdot, \cdot, \cdot, l+1)||_{2,\nu\times \mathcal{P}_\Theta}$$ $$\leq \sqrt{C}||f_{l+1}(\cdot, \cdot, \cdot, l+1) - (\mathcal{T}_l(\cdot))f_l(\cdot, \cdot, \cdot, l)||_{2,\mu\times\mathcal{P}_\Theta} + ||f_l(\cdot,\cdot,\cdot,l) - Q^*(\cdot,\cdot,\cdot,l)||_{2, P(\nu)\times \pi_{f_l,Q^*}(\cdot,\cdot, l)\times \mathcal{P}_\Theta},$$ for all $l$. We bound the first term: $$||f_{l+1}(\cdot, \cdot, \cdot, l+1) - (\mathcal{T}_l(\cdot))f_l(\cdot, \cdot, \cdot, l)||^2_{2,\mu\times\mathcal{P}_\Theta} = \mathbb{E}_\mathcal{D}[\mathcal{L}_l(f_{l+1}; f_l)] - \mathbb{E}_\mathcal{D}[\mathcal{L}_l((\mathcal{T}_{l}(\cdot))f_{l}; f_l)]$$ $$= \mathbb{E}_\mathcal{D}[\mathcal{L}_l(f_{l+1}; f_l)- \mathcal{L}_l(g^*_f(l); f)] + \mathbb{E}_\mathcal{D}[\mathcal{L}_l(g^*_f(l); f) - \mathcal{L}_l((\mathcal{T}_{l}(\cdot))f_{l}; f_l)],$$ which, by the previous lemma and taking $\mathcal{G} = \mathcal{F}$ is, with probability at least $1-\delta$, at most $$\frac{56L^2\ln\left(\frac{L|\mathcal{F}|^2}{\delta}\right)}{3m} + \sqrt{\frac{32L^2\ln\left(\frac{L|\mathcal{F}|^2}{\delta}\right)}{m}\epsilon_{\mathcal{F}, \mathcal{F}}} + \epsilon_{\mathcal{F}, \mathcal{F}},$$ so, by repeating the above recursion, we have that $$||f_{l+1}(\cdot, \cdot, \cdot, l+1) - Q^*(\cdot, \cdot, \cdot, l+1)||_{2,\nu\times \mathcal{P}_\Theta}$$ $$\leq (L+1)\sqrt{C\left(\frac{56L^2\ln\left(\frac{L|\mathcal{F}|^2}{\delta}\right)}{3m} + \sqrt{\frac{32L^2\ln\left(\frac{L|\mathcal{F}|^2}{\delta}\right)}{m}\epsilon_{\mathcal{F}, \mathcal{F}}} + \epsilon_{\mathcal{F}, \mathcal{F}}\right)},$$ for all $l$, with probability at least $1-\delta$, which, by Lemma A.1.3 implies $$v^L_{\pi^*} - v^L_{\pi_f}\leq L(L+1)\sqrt{C\left(\frac{56L^2\ln\left(\frac{L|\mathcal{F}|^2}{\delta}\right)}{3m} + \sqrt{\frac{32L^2\ln\left(\frac{L|\mathcal{F}|^2}{\delta}\right)}{m}\epsilon_{\mathcal{F}, \mathcal{F}}} + \epsilon_{\mathcal{F}, \mathcal{F}}\right)},$$ with probability at least $1-\delta$. 
\end{proof}
\noindent
The corollary regarding model-based learning immediately follows by constructing $\mathcal{F}$ from $\mathcal{H}$.

\end{document}